\documentclass[journal]{IEEEtran}
\pdfoutput=1
\IEEEoverridecommandlockouts

\usepackage{color}
\usepackage{bm}
\usepackage{hyperref}
\usepackage{textcomp}
\usepackage{color}
\usepackage{amsthm}
\usepackage{amsfonts}
\usepackage{times,amsmath}
\usepackage{amssymb}
\usepackage{graphicx}
\usepackage[dvipsnames,svgnames,x11names,table]{xcolor}
\usepackage{subcaption}
\usepackage{cite}
\usepackage{url}
\usepackage[capitalise]{cleveref}
\usepackage{enumitem}
\usepackage{booktabs}
\usepackage{tikz, pgfplots, pgfplotstable}
\usepackage[ruled,linesnumbered,noend]{algorithm2e}
\usepackage{moresize} 
\usepackage[normalem]{ulem}
\crefname{algocf}{Algorithm}{algorithms}

\input{mysymbol.sty}

\usetikzlibrary{shapes,arrows}
\pgfplotsset{compat=1.17}
\pgfplotstableset{col sep=semicolon}
\usepgfplotslibrary{groupplots}
\usepgfplotslibrary{fillbetween}

\tikzset{every mark/.append style={scale=1.6, solid}, font=\small}
\pgfplotsset{
    width=1\textwidth,
    legend style={
        font=\scriptsize ,  
        inner xsep=1pt,
        inner ysep=1pt,
        nodes={inner sep=1pt}},
    legend cell align=left,
    every axis/.append style={line width=.5pt},
 	every axis plot/.append style={line width=1.5pt},
 	every axis y label/.append style={yshift=-4pt}
}

\begin{document}

\title{Exploiting Non-Negativity in DAG \\Structure Learning}
\author{Samuel Rey, \textit{Member, IEEE,} Madeline Navarro, \textit{Graduate Student Member, IEEE,}  Santiago Segarra, \textit{Senior Member, IEEE,} and Gonzalo Mateos, \textit{Senior Member, IEEE}
\thanks{This work was supported by the NSF under Awards ECCS 2231036 and CCF 2340481, by the Spanish AEI Grants PID2022-136887NB-I00 and PID2023-149457OB-I00, and by the Community of Madrid (via grants CAM-URJC F1180 (CP2301), TEC-2024/COM-89, and Madrid ELLIS Unit).

Samuel Rey is with the Dept. of Signal Theory and Communications,
Universidad Rey Juan Carlos, Madrid, Spain (e-mail: samuel.rey.escudero@urjc.es). 
Madeline Navarro is with the Dept. of Electrical and Computer Engineering, Northeastern University, Boston, MA, USA (e-mail: m.navarro@northeastern.edu).
Santiago Segarra is with the Dept. of Electrical and Computer Engineering, Rice University, Houston, TX, USA (e-mail: segarra@rice.edu). Gonzalo Mateos is with the Dept. of Electrical and Computer Engineering, University of Rochester, Rochester, NY, USA (e-mail: gmateosb@ece.rochester.edu).}
}


\maketitle

\begin{abstract}
This work addresses the problem of learning directed acyclic graphs (DAGs) from nodal observations adhering to a linear structural equation model. DAG structure learning is a central task in signal processing, machine learning, and causal inference, but it remains challenging because acyclicity is a global combinatorial property and the search space grows superexponentially in the number of nodes. Continuous acyclicity characterizations have led to important algorithmic advances by replacing the discrete DAG restriction with smooth equality constraints on the candidate adjacency matrix. However, existing formulations still involve difficult non-convex optimization landscapes and may suffer from degenerate first-order optimality conditions. Here, we restrict attention to DAGs with non-negative edge weights and exploit this additional structure to obtain a simpler acyclicity characterization devoid of Hadamard products.
We propose a new non-negative constrained DAG estimator and develop a tractable, locally convergent algorithm based on the method of multipliers. We further analyze the benign optimization landscape induced by the non-negativity constraints. In the large-sample regime, we show that the true DAG is the unique global minimizer of the population augmented Lagrangian. Moreover, the landscape contains no spurious interior stationary points, and the true DAG is the only acyclic Karush--Kuhn--Tucker (KKT) point.
Critically, this allows us to guarantee local convergence to the true DAG given infinite samples. Comprehensive numerical experiments show that the proposed structure identification method improves over state-of-the-art continuous DAG-learning baselines.
\end{abstract}

\begin{IEEEkeywords}
DAG learning, topology inference, causal discovery, landscape analysis, linear structural equation model.
\end{IEEEkeywords}

\section{Introduction}
\IEEEPARstart{D}{\lowercase{irected}} acyclic graphs (DAGs) provide a natural language to represent dependencies in complex systems, where edge directionality is an essential part of the model~\cite{marques2020digraphs}.
DAGs are central to Bayesian networks~\cite{koller2009probabilistic,peters2017elements,spirtes2001causation}, and have become standard tools in biology~\cite{sachs2005causal,lucas2004bayesian}, genetics~\cite{zhang2013integrated}, machine learning~\cite{yu2019dag,rey2026directed}, signal processing~\cite{seifert2023causal,misiakos2025learning,mihal2025tsp}, and causal inference~\cite{spirtes2001causation,yao2021survey}.
In many applications, however, the DAG structure is not available in advance and must be inferred given nodal observations.
From the broader graph signal processing perspective, this task falls under graph learning, whose goal is to infer network structure from the statistical or structural properties of observed signals~\cite{mateos2019connecting}.
Prominent approaches for \emph{undirected} graphs invoke Gaussian assumptions~\cite{friedman2008sparse,navarro2024fair}, graph smoothness~\cite{kalofolias2016learn,dong2016learning,saboksayr2021accelerated}, or graph stationarity~\cite{segarra2017network,navarro2024joint,buciulea2025polynomial}, among others. Our focus here is on DAG recovery, where one typically postulates a structural equation model (SEM) and estimates the topology by minimizing a suitable score function; see~\cite{zheng2018dags,bello2022dagma,wei2020dags} and Section \ref{ssec:prob_statement}.
On top of directionality, enforcing acyclicity introduces a combinatorial constraint that renders such score-based formulations NP-hard to optimize~\cite{chickering96np}.
Approximate methods resort to greedy and other heuristic search strategies~\cite{maxwell1997efficient,chickering2002optimal,chickering2004large,vowels2022d}. Alternative recent schemes first seek to identify the DAG's topological order and then estimate edge weights given that order~\cite{reisach2021beware,rolland2022score,gao2022optimal,zantedeschi2023dag,deng2023optimizing,ajorlou2026build}.

\noindent\textbf{Related work and open issues.} A major breakthrough came with NOTEARS~\cite{zheng2018dags}, which replaced the discrete acyclicity requirement with a smooth equality constraint based on the (adjacency) matrix exponential, thereby unlocking the adoption of continuous optimization tools (Section \ref{ssec:continuous}).
Subsequent refinements include alternative acyclicity functions based on matrix polynomials~\cite{yu2019dag,wei2020dags}, log-determinant characterizations~\cite{bello2022dagma}, and likelihood-based formulations~\cite{ng2020role}.
Beyond static linear SEMs, the continuous framework has also been extended to jointly estimate the DAG and the exogenous noise levels~\cite{saboksayr2023colide,mateos2026colide}, to time-series models~\cite{pamfil2020dynotears,misiakos2025learning}, and to nonlinear mechanisms~\cite{yu2019dag,misiakos2026learning}.
Most continuous acyclicity functions involve Hadamard products to square all adjacency matrix entries, thus avoiding missed cycles due to potential catastrophic cancellation of sign-indefinite weights over closed paths.
This composition with the Hadamard product leads to challenging non-convex optimization landscapes and causes the constraint gradient to vanish at every feasible DAG, yielding degenerate first-order optimality conditions~\cite{wei2020dags}.


From a theoretical perspective, population-level analyses provide conditions under which the true DAG is identifiable as the unique minimizer of a suitable score~\cite{loh2014high}.
For continuous formulations, existing results characterize the relation between Karush--Kuhn--Tucker (KKT) points and local minima~\cite{wei2020dags}, establish convergence of quadratic-penalty methods to a DAG solution under mild conditions~\cite{ng2022convergence}, and examine how non-convexity, noise assumptions, sparsity, and thresholding affect practical performance~\cite{ng2024sober}. For a recent population landscape analysis in the bivariate setting, see~\cite{deng2023global}. Despite these advances, guarantees that continuous algorithms recover the underlying DAG remain elusive, largely due to the intricacies of the non-convex optimization landscape.

\noindent\textbf{Proposed approach and contributions.} In this work, 
we impose \emph{additional structure} on the problem and focus on DAGs with \emph{non-negative edge weights} (Section \ref{sec:nonnegative_dag_learning}). This assumption is appropriate when directed influences are known to be sign-definite, as in excitatory neural or economic influence networks, certain flow, pollution, or contagion models, and in systems where inhibitory effects are absent or can be separated beforehand. Imposing non-negativity excludes cancellations that may otherwise arise with matrix powers, thereby enabling acyclicity to be characterized by a function applied directly to the adjacency matrix. This yields a smooth acyclicity constraint whose gradient does not vanish at DAGs, thus avoiding the KKT degeneracy that affects several existing continuous formulations~\cite{wei2020dags}.
Building on this observation, we formulate a novel non-negative DAG learning estimator based on a regularized least-squares (LS) score and a customized log-determinant acyclicity constraint.
We solve the resulting constrained optimization problem using the method of multipliers (MoM), 
leading to an algorithm we call \emph{Non-negative Optimization via Multipliers for Acyclic Digraphs (NOMAD)}.

Although the optimization problem remains non-convex, the proposed formulation has a favorable population landscape.
In particular, under standard identifiability assumptions, in Section \ref{ssec:landscape} we show that the true DAG is the unique global minimizer of the population augmented Lagrangian, that there are no spurious interior stationary points, and that the only acyclic KKT point is the true DAG. We then establish that NOMAD's projected gradient iterations are locally convergent to the true DAG (Section \ref{ssec:convergence}).
Our reproducible numerical results in Section \ref{sec:numerical_evaluation} further indicate that this favorable population behavior carries over to finite-sample regimes, showing how the estimated graph approaches the ground truth as the number of observations increases. Moreover, NOMAD compares favorably with established DAG-learning baselines, even when these are modified to account for non-negativity.

To summarize, our significant contributions are as follows.
\begin{itemize}
    \item For DAGs with non-negative edge weights, we prove that acyclicity can be characterized directly from the adjacency matrix (Proposition \ref{prop:nonneg_logdet}), avoiding Hadamard products and the resulting degeneracy of the KKT system.
        
    \item We propose NOMAD, a locally convergent (\cref{thm:convergence}) MoM algorithm to learn non-negative weighted DAGs.
    
    \item We expose a favorable population landscape. When identifiable, the true DAG is the unique global minimizer of the population augmented Lagrangian (Theorem \ref{thm:population_unique_minimizer}) and the sole acyclic KKT point (Theorem \ref{thm:acyclic_kkt_points}). Also, there are no spurious interior stationary points (Theorem \ref{thm:population_stationary_points}).
    
    \item Experiments corroborate that NOMAD exhibits finite-sample behavior consistent with the population analysis and it outperforms state-of-the-art alternatives. We share the code to reproduce all results reported in \cref{sec:numerical_evaluation}.
\end{itemize}

\noindent Relative to the conference precursor~\cite{rey2025non}, here we study non-negative DAG learning through a unified presentation along with full technical details (including extended discussions and unpublished theoretical results with their proofs), supported by more comprehensive experiments with expanded baselines.



\section{Fundamentals of DAG Structure Learning}\label{sec:fundamentals}
This section introduces fundamental concepts related to DAGs and formally states the DAG structure learning problem.

\subsection{DAGs, Structural Equation Models, and Problem Statement}\label{ssec:prob_statement}

Let $\ccalD=(\ccalV,\ccalE)$ be a directed graph with node set $\ccalV=\{1,\ldots,d\}$ and directed edge set $\ccalE\subseteq\ccalV\times\ccalV$. 
An edge $(i,j)\in\ccalE$ represents a directed link from node $i$ to node $j$.
The graph topology is encoded by a weighted adjacency matrix $\bbW\in\reals^{d\times d}$, whose entries satisfy $W_{ij}\neq 0$ if and only if $(i,j)\in\ccalE$.
A graph is acyclic if it contains no directed cycles.
Throughout, we let $\mbD$ denote the set of matrices in $\reals^{d\times d}$ whose support corresponds to a DAG.

DAG structure learning aims to infer the unknown connectivity encoded in $\bbW$ from nodal observations.
Let $\bbX := [\bbx_1,\ldots,\bbx_n]\in\reals^{d\times n}$ collect $n$ observed graph signals, with $\bbx_i \in \reals^d$ denoting the $i$-th sample. We postulate a linear SEM
\begin{equation}
    \bbX = \bbW^\top \bbX + \bbZ,
    \label{eq:sem}
\end{equation}
where $\bbZ\in\reals^{d\times n}$ collects exogenous noise samples.
Unless otherwise stated, the columns of $\bbZ$ are assumed independent and identically distributed (i.i.d.) random vectors with scaled-identity covariance matrix $\bbSigma_{\bbz} = \sigma^2\bbI$.
Independence of the exogenous variables is a crucial requirement for causal interpretation and identifiability in SEMs~\cite{loh2014high}, \cite[pp. 83--84]{peters2017elements}.

Given $\bbX$, the DAG structure encoded in $\bbW$ can be inferred from the observed data by solving the optimization problem
\begin{equation}\label{eq:dag_learning} 
    \min_{\bbW}  F \left( \bbW; \bbX \right) \quad \mathrm{s. to} \quad \bbW \in \mbD,
\end{equation}
where $F(\bbW;\bbX)$ is a data-dependent score (i.e., loss) function that captures the relation between $\bbW$ and the observations.
The difficulty in \eqref{eq:dag_learning} lies in the feasible set $\mbD$.
Directly enforcing acyclicity is a combinatorial endeavor and exact score-based DAG learning is NP-hard in general~\cite{maxwell1997efficient,chickering2004large}.

\subsection{Continuous Acyclicity Constraints}\label{ssec:continuous}
Modern approaches to DAG structure learning advocate for replacing the discrete constraint $\bbW\in\mbD$ with a smooth equality constraint $h(\bbW)=0$, where $h:\reals^{d\times d}\to\reals$ is a differentiable function whose zero level set coincides with $\mbD$.
That is, the acyclicity function $h$ satisfies $h(\bbW)=0$ if and only if $\bbW \in \mbD$.
The discrete constraint in \eqref{eq:dag_learning} can then be replaced by the continuous formulation
\begin{equation}\label{eq:dag_learning_cont} 
    \min_{\bbW}  F \left( \bbW; \bbX \right) \quad \mathrm{s. to} \quad h(\bbW) = 0.
\end{equation}
This way, DAG learning is cast as a continuous constrained optimization problem amenable to first-order methods.

Relying on smooth functions to enforce acyclicity was pioneered in NOTEARS~\cite{zheng2018dags}, which proposed
\begin{equation}\label{eq:notears}
	h_{\mathrm{notears}}(\bbW) = \tr \left( e^{\bbW \circ \bbW} \right) - d,
\end{equation}
where $\circ$ denotes the Hadamard (entry-wise) product.
Subsequent works relied on alternative smooth characterizations based on matrix powers~\cite{wei2020dags,pamfil2020dynotears} and log-determinants.
In particular, the DAGMA estimator in~\cite{bello2022dagma} considers
\begin{equation}\label{eq:dagma}
		h_{\mathrm{dagma}}^s(\bbW) = d \log (s) - \log\det \left( s\bbI - \bbW \circ \bbW \right),
\end{equation}
with $s\in\reals_+$ chosen so that the log-determinant is well defined.
This barrier-type characterization has favorable gradient properties and has been shown to perform well in practice.

Continuous acyclicity constraints have substantially changed the algorithmic landscape of DAG structure learning.
Replacing the combinatorial constraint $\bbW \in \mbD$ with the acyclicity condition $h(\bbW) = 0$ unlocks the adoption of standard continuous optimization methods.
This paradigm shift has led to scalable algorithms with competitive empirical performance in settings where exhaustive combinatorial search is infeasible.
However, the resulting optimization problems remain fundamentally challenging.
In particular, the term $\bbW \circ \bbW$ introduces a non-convex composition that complicates the optimization landscape.
Moreover, this term makes the gradient of the acyclicity constraint vanish at every DAG, i.e., $\nabla h(\bbW) = \bm{0}$ for every $\bbW \in \mbD$.
As observed in~\cite{wei2020dags}, this leads to degenerate KKT conditions (we drop the data dependency in $F$ for notational convenience)
\begin{equation}\label{eq:kkt}
    \nabla F(\bbW) + \lambda \nabla h(\bbW) = \bbzero,
\end{equation}
since a DAG $\bbW \in \mbD$ can be a KKT point of \eqref{eq:dag_learning_cont} only if it is an unconstrained stationary point of the score function, i.e., $\nabla F(\bbW) = \bbzero$.
Consequently, when $F$ is convex as in LS, this condition holds only at unconstrained global minimizers of $F$.

To overcome these limitations, we henceforth assume that $\bbW$ \emph{has non-negative weights} and propose a simpler acyclicity function devoid of Hadamard products. This yields a more amenable optimization landscape, avoids the aforementioned degeneracy at feasible DAGs, and opens the door to stronger guarantees for the global behavior of the resulting estimator.

\section{Non-Negative DAG Structure Learning}\label{sec:nonnegative_dag_learning}

We specialize the DAG structure learning problem to the case where the edge weights are non-negative.
This restriction is natural in settings where directed interactions encode excitatory, additive, or non-inhibitory relations, and subsumes binary adjacency matrices; see e.g.,~\cite{sachs2005causal,seifert2023causal}.
From an optimization viewpoint, non-negativity precludes sign cancellations that could otherwise arise in matrix powers. We show acyclicity can be characterized using simpler smooth functions that do not require the Hadamard product safeguard.

When the observations follow the linear SEM in \eqref{eq:sem}, a standard choice is to estimate the adjacency matrix by combining an LS data-fitting term with a sparsity-promoting penalty. Since $\bbW$ is constrained to be entrywise non-negative, the $\ell_1$-norm penalty reduces to a linear term. This yields the estimator
\begin{alignat}{2}
    \hbW =
    \operatorname*{arg\,min}_{\bbW}
    \quad & \left\{
    \frac{1}{2n}\| \bbX-\bbW^\top\bbX \|_F^2
    +
    \alpha_n \sum_{i\neq j} W_{ij} \right\}
    \label{eq:nonneg_dag_learning}
    \\
    \mathrm{s. to}
    \quad &
    \bbW\geq \bbzero,\qquad
    h(\bbW)=0, \nonumber
\end{alignat}
where $\alpha_n>0$ is a tunable parameter that controls the trade-off between data fidelity and sparsity, chosen so that $\alpha_n\to 0$ as $n \to \infty$. For instance, minimax optimality in sparse regression calls for $\alpha_n\asymp\sigma\sqrt{\log(d)/n}$; see e.g.,~\cite{saboksayr2023colide}.
Moreover, $h$ is a smooth acyclicity function that we specify in the ensuing section. The diagonal entries are excluded from the sparsity penalty since self-loops are incompatible with acyclicity. Equivalently, one may explicitly impose $\diag(\bbW)=\bbzero$ if desired.

\subsection{Acyclicity over Non-Negative Matrices}

Smooth acyclicity constraints are central to modern continuous optimization approaches for DAG learning. As discussed in~\cite{zheng2018dags}, an effective acyclicity function $h(\bbW)$ should be smooth, have an efficiently computable gradient, and satisfy $h(\bbW)=0$ if and only if $\bbW\in\mbD$. By exploiting the DAG non-negativity and drawing inspiration from the log-determinant characterization in \eqref{eq:dagma}, we introduce a simpler function that meets these requirements without Hadamard products.
\begin{definition}\label{def:nonneg_logdet}
    For any  $s\in\reals_+$, define the function
    \begin{equation}\label{eq:nonneg_logdet_constraint}
        h(\bbW) := d\log(s)-\log\det(s\bbI-\bbW),
    \end{equation}
    whose domain is the set of matrices $\bbW \in \reals^{d \times d}$ with spectral radius $\rho(\bbW)<s$.
    Its gradient is $\nabla h(\bbW) = (s\bbI-\bbW)^{-\top}$.
    %
    %
\end{definition}
The bounded spectral radius guarantees that the log determinant is well-defined.
Interestingly, the gradient of $h(\bbW)$ introduced in \cref{def:nonneg_logdet} does not vanish when $\bbW\in\mbD_+$, with $\mbD_+$ denoting the set of DAGs with non-negative edge weights.
Next, we formally show that the proposed function characterizes acyclicity over non-negative matrices.
\begin{proposition}\label{prop:nonneg_logdet}
    Let  $s\in\reals_+$ and consider the set of matrices
    \begin{equation}\label{eq:nonneg_feasible_set}
        \ccalW_s := \left\{\bbW\in\reals_+^{d\times d}:\rho(\bbW)<s\right\}.
    \end{equation}
    Then, $h$ in \eqref{eq:nonneg_logdet_constraint} is a valid acyclicity function over $\ccalW_s$, where $\mbD_+ \! \subset \! \ccalW_s$. Specifically, for
    $\bbW \!\in\! \ccalW_s$, it satisfies $h(\bbW) \!\geq\! 0$ and
    \begin{equation}
        h(\bbW) = 0 \quad \Longleftrightarrow \quad \bbW \in \mbD_+.
    \end{equation}
    %
\end{proposition}

\begin{proof}
First, rewrite $h$ as
\begin{align}
    h(\bbW)
    &=
    d\log(s)-\log(s^d)-\log\det(\bbI-s^{-1}\bbW)
    \nonumber\\
    &=
    -\log\det(\bbI-s^{-1}\bbW).
    \label{eq:ldet_rescaled}
\end{align}
Since $\rho(s^{-1}\bbW)<1$, the matrix logarithm is well defined, and applying the Mercator series gives 
\begin{equation}
    \log(\bbI-s^{-1}\bbW) =
    -\sum_{k=1}^{\infty}\frac{s^{-k}\bbW^k}{k}.
\end{equation}
Then, using the identity $\log\det(\bbM)=\tr(\log\bbM)$, we obtain
\begin{align}
    h(\bbW) &=
    \tr\left(
    \sum_{k=1}^{\infty}\frac{s^{-k}\bbW^k}{k}
    \right) 
    =
    \sum_{k=1}^{\infty}\frac{\tr(\bbW^k)}{k s^k}.
    \label{eq:ldet_trace_series}
\end{align}
Because $\bbW$ is entrywise non-negative, every diagonal entry of $\bbW^k$ is non-negative, and hence $\tr(\bbW^k)\geq 0$ for all $k\geq 1$.
It follows from \eqref{eq:ldet_trace_series} that $h(\bbW)\geq 0$.

It remains to characterize the zero level set.
Since all terms in \eqref{eq:ldet_trace_series} are non-negative, $h(\bbW)=0$ holds if and only if $\tr(\bbW^k)=0$ for every $k\geq 1$.
The diagonal entry $[\bbW^k]_{ii}$ is the total weight of all directed closed walks of length $k$ that start and end at node $i$.
Thus, if the graph associated with $\bbW$ contained a directed cycle, the trace of a suitable power of $\bbW$ would be strictly positive.
Therefore, $h(\bbW)=0$ implies that $\bbW$ represents a DAG with non-negative edge weights.
Conversely, if $\bbW\in\mbD_+$, then $\bbW^k$ is strictly upper triangular up to a common permutation of rows and columns, so $\tr(\bbW^k)=0$ for every $k\geq 1$ and \eqref{eq:ldet_trace_series} gives $h(\bbW)=0$.
\end{proof}

The non-negativity assumption is essential in \cref{prop:nonneg_logdet}.
Indeed, the series in \eqref{eq:ldet_trace_series} expresses $h$ as a weighted sum of closed-walk contributions.
When edge weights are non-negative, every directed cycle contributes a non-negative term and cannot be masked by cancellations from other closed walks.
As a result, the zero level set of $h$ coincides with the set $\mbD_+$ of non-negative DAGs. This cancellation-free property is precisely what is lost when sign-indefinite weights are allowed.

Although $h$ is not convex, it retains a favorable first-order structure that is absent from Hadamard-based acyclicity functions.
In particular, $\nabla h$ does not vanish at DAGs, so feasible acyclic matrices are not automatically stationary points of $h$, hence preventing the degeneracy of the KKT conditions in \eqref{eq:kkt}.

\subsection{Algorithmic Implementation via the Method of Multipliers}

We solve \eqref{eq:nonneg_dag_learning} using the MoM, an augmented-Lagrangian approach for equality-constrained optimization~\cite[Ch. 4.2]{bertsekas2016nonlinear}, and refer to the resulting algorithm as NOMAD.
The appeal of the MoM is that, in settings where the required regularity and convexity assumptions hold, minimizing a sequence of augmented Lagrangians recovers a solution of the original constrained problem.
In the present setting, the acyclicity constraint remains non-convex, so these classical guarantees cannot be invoked directly.
Nevertheless, the MoM provides a principled way to drive the acyclicity violation to zero while retaining a smooth objective at each iteration.

Let us denote the score function in 
\eqref{eq:nonneg_dag_learning} by $F_n(\bbW)$.
%
For a multiplier $\lambda \in \reals$ and a penalty parameter $c>0$, we consider the augmented Lagrangian function
\begin{equation}\label{eq:augmented_lagrangian}
    L_c(\bbW,\lambda)
    :=
    F_n(\bbW)
    +
    \lambda h(\bbW)
    +
    \frac{c}{2}h^2(\bbW).
\end{equation}
Only the acyclicity constraint is dualized and quadratically penalized in~\eqref{eq:augmented_lagrangian}.
The non-negativity constraint is kept explicit, since it can be handled directly by projection onto the non-negative orthant.
Similarly, the spectral-radius requirement is treated as a domain restriction for the log-determinant term.

Starting from $\lambda^{(0)}$ and $c^{(0)}>0$, the MoM repeats three steps at each iteration $k$.

\noindent\textbf{Step 1.}
The adjacency matrix is updated by solving
\begin{equation}\label{eq:al_w_update}
    \bbW^{(k+1)} =
    \operatorname*{arg\,min}_{\bbW \in \ccalW_s}
    L_{c^{(k)}}(\bbW,\lambda^{(k)}),
\end{equation}
where $\ccalW_s$ is the set defined in \eqref{eq:nonneg_feasible_set}.
In practice, this inner problem can be addressed with projected first-order iterations.
For the log-determinant acyclicity function in \cref{def:nonneg_logdet}, the gradient of the augmented Lagrangian is given by
\begin{equation}\label{eq:grad_augmented_lagrangian}
    \nabla_{\bbW}L_c(\bbW,\lambda)
    =
    \nabla F_n(\bbW)
    +
    \left(\lambda + c h(\bbW)\right) \nabla h(\bbW). 
\end{equation}
The gradient of the score function is
\begin{equation}\label{eq:grad_score}
    \nabla F_n(\bbW)
    =
    \hbSigma_{\bbx}(\bbW-\bbI)
    +
    \alpha_n(\bbone\bbone^\top-\bbI),
\end{equation}
where $\hbSigma_{\bbx} = n^{-1}\bbX\bbX^\top$ is the sample covariance matrix and $\bbone$ is the all-ones vector.
Warm started as $\bbW_{(0)} = \bbW^{(k)}$, projected gradient descent (PGD) for \eqref{eq:al_w_update} with stepsize $\eta>0$ reads
\begin{equation}\label{eq:projected_gradient_step}
    \bbW_{(\ell+1)}
    =
    \left[
    \bbW_{(\ell)}
    -
    \eta
    \nabla_{\bbW}L_{c^{(k)}}(\bbW_{(\ell)},\lambda^{(k)})
    \right]_+.
\end{equation}
Here $[\cdot]_+:=\max(\cdot,0)$ denotes the entrywise projection onto the non-negative orthant.
The stepsize $\eta$ is chosen so that the iterates remain in the domain $\rho(\bbW_{(\ell)})<s$ of the log-determinant term.
Once the inner stopping criterion is met, the updated outer iterate is set to the final inner iterate, i.e., $\bbW^{(k+1)} = \bbW_{(\ell+1)}$.

\noindent\textbf{Step 2.}
The Lagrange multiplier is updated according to the current constraint violation,
\begin{equation}\label{eq:al_lambda_update}
    \lambda^{(k+1)}
    =
    \lambda^{(k)}
    +
    c^{(k)}h(\bbW^{(k+1)}).
\end{equation}
This update can also be interpreted as a dual gradient ascent step since the constraint violation corresponds to the gradient of $L_{c^{(k)}}(\bbW^{(k+1)},\lambda)$ with respect to $\lambda$.

\noindent\textbf{Step 3.}
The penalty parameter is updated as 
%
\begin{equation}\label{eq:al_c_update}
    c^{(k+1)}
    =
    \begin{cases}
        \beta c^{(k)}, &
        \text{if } h(\bbW^{(k+1)})>\gamma h(\bbW^{(k)}),\\
        c^{(k)}, &
        \text{otherwise},
    \end{cases}
\end{equation}
where $\beta>1$ and $0<\gamma<1$.
This adaptive rule increases the penalty only when the current value of $c^{(k)}$ is not producing a sufficient decrease in the acyclicity residual.

\begin{algorithm}[t]
\caption{NOMAD}
\label{alg:nomad}
\KwIn{$\bbX$, $\alpha_n$, $s$, $\eta$, $\beta$, $\gamma$, $\bbW^{(0)}$, $\lambda^{(0)}$, and $c^{(0)}$}
\For{$k=0,1,\ldots$ until outer criterion is met}{
    $\bbW_{(0)} =\bbW^{(k)}$\;
    \For{$\ell=0,1,\ldots$ until inner criterion is met}{
        $\bbW_{(\ell+1)} =
        \left[\bbW_{(\ell)}-\eta\nabla_{\bbW}
        L_{c^{(k)}}(\bbW_{(\ell)},\lambda^{(k)})\right]_+$\;
    }
    $\bbW^{(k+1)} =\bbW_{(\ell+1)}$\;
    $\lambda^{(k+1)} =
    \lambda^{(k)}+c^{(k)}h(\bbW^{(k+1)})$\;
    Update $c^{(k+1)}$ via \eqref{eq:al_c_update}\;
}
\Return{$\hbW = \bbW^{(k)}$}
\end{algorithm}

The overall NOMAD procedure is given in \cref{alg:nomad}.
It follows the MoM rationale; namely, the quadratic term penalizes violations of $h(\bbW)=0$, while the multiplier update adjusts the linear penalty according to the current constraint violation.
The computational cost is dominated by the inner problem in \eqref{eq:al_w_update}.
Each projected-gradient step evaluates \eqref{eq:grad_augmented_lagrangian}, at a cost of $O(d^3)$ due to the computation of $(s\bbI-\bbW)^{-\top}$ and the associated matrix products.
In practice, we initialize $\bbW^{(0)}$ to the zero matrix and observe that few outer iterations suffice to satisfy the acyclicity constraint.
We further observe that NOMAD is relatively insensitive to the value of $s$.
Thus, unlike other log-determinant-based DAG learning algorithms that prescribe a sequence of values for this parameter~\cite{bello2022dagma,saboksayr2023colide}, we fix $s=1$ throughout.
Finally, while \cref{alg:nomad} uses projected-gradient iterations for simplicity of exposition and tractability in our convergence analysis, implementations based on Adam~\cite{kingma2015adam} or FISTA~\cite{beck2009fast} require fewer inner iterations in practice, thereby yielding more competitive computational performance; see also Section \ref{ssec:synthetic}.

The non-convexity of $h$ prevents us from invoking the classical convergence theory for augmented Lagrangian methods, which requires exact minimization of \eqref{eq:al_w_update} at each outer iteration.
Nevertheless, the numerical results in \cref{sec:numerical_evaluation} suggest that the estimated graph approaches the true DAG as the sample size increases.
To shed light on this promising behavior, we next study the optimization landscape of the problem at hand. In particular, we show that when $n\to\infty$, and thus $\hbSigma_{\bbx}\to \bbSigma_{\bbx}= \E{\bbx\bbx^\top}$ and $\alpha_n\to 0$, the inner PGD iterations \eqref{eq:projected_gradient_step} converge locally to the ground-truth DAG, implying that NOMAD can do so too.

\section{Landscape and Convergence Analyses}\label{sec:optimization_landscape}


We elucidate how non-negative weights and the simpler acyclicity constraint yield a benign landscape for the \emph{population} augmented Lagrangian, enabling local convergence. 

\subsection{Landscape of the Population Augmented Lagrangian}\label{ssec:landscape}


Let $\bbW_0$ denote the weighted adjacency matrix of the true non-negative DAG, and let $\bbx$ be an observation from the SEM in \eqref{eq:sem}.
At the population level, the empirical LS term is replaced by its expectation, leading to the score function
\begin{equation}\label{eq:population_score}
    \bar{F}(\bbW) \!:=\!
    \!\frac{1}{2} \E{\!
        \| \bbx \!-\!\! \bbW^\top \!\bbx \|_2^2} \!=\!
    \frac{1}{2}\tr\!\left(
        (\bbI \!-\!\! \bbW)^\top \!
        \bbSigma_{\bbx}
        \!(\bbI \!-\!\! \bbW)
    \right).
\end{equation}
Since \eqref{eq:population_score} is defined in the population regime (i.e., $n\to\infty$), no sparsity-promoting term is included ($\alpha_n = 0$).

For linear SEMs the identifiability of $\bbW_0$ is tied to the exogenous noise distribution.
In the analysis that follows, we focus on the normalized case where $\bbSigma_{\bbz} = \E{\bbz\bbz^\top} = \bbI$ (cf. Assumption \ref{ass:population_sem}).
Extensions to homoscedastic noise with covariance $\sigma^2\bbI$ or heteroscedastic noise with diagonal covariance known up to a scalar factor follow directly by rescaling.
Crucially, in all aforementioned renditions of the noise model, the ground-truth DAG is identifiable in the population regime~\cite{loh2014high,peters2017elements}.

From \eqref{eq:population_score} and the acyclicity constraint in \cref{prop:nonneg_logdet}, the population augmented Lagrangian is given by
%
\begin{equation}\label{eq:population_augmented_lagrangian}
    \bar{L}_c(\bbW,\lambda) :=
    \bar{F}(\bbW) + \lambda h(\bbW) + \frac{c}{2}h^2(\bbW).
\end{equation}
The results we derive next characterize the global and first-order behavior of \eqref{eq:population_augmented_lagrangian} over the non-negative domain \eqref{eq:nonneg_feasible_set} of the log-determinant acyclicity function. We set $s=1$, corresponding to the function $h(\bbW)=-\log\det(\bbI-\bbW)$, whose domain is $\ccalW_1$. We adopt the following assumption.

\begin{assumption}[Population SEM]\label{ass:population_sem}
The data follow the linear SEM in \eqref{eq:sem}, where
$\bbW_0\in\ccalW_1$ is the adjacency matrix of the true DAG and the exogenous noise satisfies $\mbE[\bbz\bbz^\top]=\bbI$.
\end{assumption}

Under Assumption \ref{ass:population_sem}, the true DAG is identifiable as $n\to\infty$ in the following sense~\cite{loh2014high}:
\begin{equation}\label{eq:population_identifiability}
    \bbW_0 = \arg\min_\bbW  \bar{F}(\bbW) \quad \mathrm{s. to} \quad \bbW\in\ccalW_1,\ h(\bbW)=0.
\end{equation}
In other words, in the (large sample) population regime, $\bbW_0$ uniquely minimizes $\bar{F}$ over all non-negative acyclic matrices since $\mbD_+\subset \ccalW_1$.
We next characterize the population augmented Lagrangian through three complementary properties: its global minimizer, its classical stationary points, and its acyclic KKT points. All proofs are in the appendices.

Our first result states that the true DAG $\bbW_0$ is the unique global minimizer of $\bar{L}_c(\cdot,\lambda)$ over the entire domain $\ccalW_1$.

\begin{theorem}[Unique global minimizer]\label{thm:population_unique_minimizer}
Under Assumption~\ref{ass:population_sem}, for every $c>0$ and $\lambda\geq 1$, $\bbW_0$ is the unique global minimizer of the population augmented Lagrangian over $\ccalW_1$, namely
\begin{equation}
\label{eq:population_problem}
    \bbW_0
    =
    \operatorname*{arg\,min}_{\bbW\in\ccalW_1}
    \bar{L}_c(\bbW,\lambda).
\end{equation}
\end{theorem}

\noindent By virtue of \cref{thm:population_unique_minimizer}, identifiability of $\bbW_0$ over the constrained DAG set is lifted to the augmented-Lagrangian objective over the larger non-negative domain $\ccalW_1$.

However, global uniqueness alone does not fully capture what first-order methods may encounter in a non-convex landscape.
We therefore turn to classical stationary points of the smooth augmented-Lagrangian objective.

\begin{theorem}[Classical stationary points]\label{thm:population_stationary_points}
Under Assumption~\ref{ass:population_sem}, let $c>0$.
Then the following statements hold:
\begin{enumerate}[label=(\roman*)]
    \item If $\lambda>1$,  $\nexists\:\bbW\in\ccalW_1$ such that $\nabla_{\bbW}\bar{L}_c(\bbW,\lambda)=\bbzero$.
    \item If $\lambda=1$, then $\nabla_{\bbW}\bar{L}_c(\bbW,1) =\bbzero \Longleftrightarrow \bbW=\bbW_0$.
    %
    %
\end{enumerate}
\end{theorem}

\noindent For $\lambda\geq 1$, \cref{thm:population_stationary_points} rules out spurious interior stationary points of the population augmented Lagrangian.
Thus, any additional first-order candidates obtained via PGD methods may arise on the boundary of the non-negative cone, which motivates the KKT analysis dealt with next.

Since the spectral radius condition in $\ccalW_1$ is with strict inequality, the only active constraint in the first-order analysis over $\ccalW_1$ is $\bbW\geq \bbzero$.
Thus, a matrix $\bbW\in\ccalW_1$ is a KKT point of the problem of minimizing $\bar{L}_c(\cdot,\lambda)$ over $\ccalW_1$ if and only if there exists a multiplier matrix $\bbGamma\in\reals_+^{d\times d}$ such that $\bbGamma\circ\bbW=\bbzero$ and $\nabla_{\bbW}\bar{L}_c(\bbW,\lambda)-\bbGamma=\bbzero$.
%
%
Equivalently,
\begin{equation}\label{eq:kkt_augmented_lagrangian_equiv}
    \nabla_{\bbW}\bar{L}_c(\bbW,\lambda)\geq \bbzero,
    \qquad
    \bbW\circ\nabla_{\bbW}\bar{L}_c(\bbW,\lambda)=\bbzero.
\end{equation}

\begin{theorem}[Acyclic KKT points]\label{thm:acyclic_kkt_points}
Under Assumption~\ref{ass:population_sem}, let $c>0$.
If $\bbW\in\ccalW_1$ is a KKT point of $\bar{L}_c(\cdot,1)$ over $\ccalW_1$ and $h(\bbW)=0$, then $\bbW=\bbW_0$.
\end{theorem}


Together, \cref{thm:population_unique_minimizer,thm:population_stationary_points,thm:acyclic_kkt_points} reveal a favorable population landscape: $\bbW_0$ is the unique global minimizer of $\bar{L}_c$, no spurious interior stationary points exist, and, for $\lambda=1$, every acyclic KKT point is the true DAG.
The proofs rely on both non-negativity and the acyclicity characterization introduced in \cref{prop:nonneg_logdet}.
To illustrate their importance, consider instead sign-indefinite DAGs with DAGMA's acyclicity function $h_{\mathrm{dagma}}^1$ in \eqref{eq:dagma}.
Since $\nabla_{\bbW}h_{\mathrm{dagma}}^1(\bbW)=\bbzero$ for every $\bbW\in\mbD$, an acyclic matrix can be a KKT point of the corresponding augmented-Lagrangian objective only if it is stationary for the population score \eqref{eq:population_score}, i.e., $\bbSigma_\bbx(\bbW-\bbI)=\bbzero$.
%
%
Thus, the true DAG is not, in general, a KKT point of this alternative formulation.
The claim of global optimality in \cref{thm:population_unique_minimizer} also fails to hold in general. For $\lambda\geq 1$ and $c>0$, one can construct instances in which a cyclic matrix attains a lower augmented Lagrangian value under $h_{\mathrm{dagma}}^1$ than $\bbW_0$.


\subsection{Local Convergence of Projected Gradient Descent}\label{ssec:convergence}

The preceding results establish that exploiting non-negativity yields a benign optimization landscape.
However, this alone does not warrant convergence to the global minimizer, since the problem remains non-convex.
Hence, we next establish conditions under which PGD applied to \eqref{eq:population_problem} converges locally to its unique global minimizer $\bbW_0$ in the population regime.

\begin{table*}[!t]
    \centering
    \caption{Comparison of DAG-learning methods on synthetic data. For all metrics, we report the mean $\pm$ standard deviation. The best and second-best values are shown in bold and underlined, respectively, based on the unrounded means.}
    \label{tab:synthetic_results}
    \scriptsize
    \setlength{\tabcolsep}{3.5pt}
    \begin{tabular}{lccccccc}
        \toprule
        Method & SHD $\downarrow$ & TPR $\uparrow$ & FDR $\downarrow$ & $F_1$ score $\uparrow$ & $\mathrm{SID}_{\mathrm{norm}}$ $\downarrow$ & nerr $\downarrow$ & Time (s) \\
        \midrule
        NOMAD-Adam & $\boldsymbol{0.01 \pm 0.10}$ & $\boldsymbol{1.00 \pm 0.00}$ & $\boldsymbol{0.00 \pm 0.00}$ & $\boldsymbol{1.00 \pm 0.00}$ & $\boldsymbol{0.01 \pm 0.11}$ & $\boldsymbol{0.00 \pm 0.00}$ & \underline{$30.63 \pm 7.87$} \\
        NOMAD-FISTA & \underline{$0.12 \pm 0.41$} & \underline{$1.00 \pm 0.00$} & \underline{$0.00 \pm 0.00$} & \underline{$1.00 \pm 0.00$} & \underline{$0.15 \pm 0.49$} & \underline{$0.00 \pm 0.00$} & $32.52 \pm 9.24$ \\
        \midrule
        DAGMA & $22.47 \pm 10.35$ & $0.96 \pm 0.02$ & $0.03 \pm 0.02$ & $0.96 \pm 0.02$ & $7.57 \pm 2.57$ & $0.06 \pm 0.02$ & $30.84 \pm 16.82$ \\
        NonDAGMA & $3.75 \pm 29.09$ & $0.99 \pm 0.04$ & $0.00 \pm 0.03$ & $0.99 \pm 0.04$ & $0.59 \pm 0.84$ & $0.01 \pm 0.06$ & $35.65 \pm 21.81$ \\
        \midrule
        CoLiDE-EV & $22.99 \pm 9.67$ & $0.96 \pm 0.02$ & $0.03 \pm 0.02$ & $0.96 \pm 0.01$ & $7.56 \pm 2.45$ & $0.06 \pm 0.02$ & $38.02 \pm 24.69$ \\
        CoLiDE-NV & $28.56 \pm 13.14$ & $0.95 \pm 0.02$ & $0.04 \pm 0.02$ & $0.95 \pm 0.02$ & $8.46 \pm 2.71$ & $0.07 \pm 0.03$ & $39.17 \pm 23.50$ \\
        NoTears & $47.87 \pm 24.12$ & $0.92 \pm 0.03$ & $0.06 \pm 0.03$ & $0.93 \pm 0.03$ & $10.84 \pm 3.07$ & $0.13 \pm 0.05$ & $3468.42 \pm 1175.69$ \\
        NoFears & $33.13 \pm 21.16$ & $0.95 \pm 0.03$ & $0.04 \pm 0.03$ & $0.95 \pm 0.03$ & $5.84 \pm 3.04$ & $0.09 \pm 0.05$ & $2475.25 \pm 427.14$ \\
        GOLEM-EV & $6.61 \pm 4.19$ & $0.99 \pm 0.01$ & $0.01 \pm 0.01$ & $0.99 \pm 0.01$ & $4.28 \pm 2.17$ & $0.02 \pm 0.01$ & $277.55 \pm 171.28$ \\
        GOLEM-NV & $7.73 \pm 4.95$ & $0.99 \pm 0.01$ & $0.02 \pm 0.01$ & $0.99 \pm 0.01$ & $4.29 \pm 2.20$ & $0.03 \pm 0.01$ & $436.67 \pm 207.83$ \\
        SortNReg. & $74.38 \pm 22.98$ & $0.96 \pm 0.01$ & $0.15 \pm 0.04$ & $0.90 \pm 0.03$ & $6.40 \pm 2.10$ & $0.12 \pm 0.04$ & $\boldsymbol{1.83 \pm 1.23}$ \\
        \bottomrule
    \end{tabular}
    \vspace*{-3mm}
\end{table*}

To this end, we restrict the feasible set in \eqref{eq:population_problem} to a compact subset containing $\bbW_0$.
Since $\rho(\bbW_0)=0$, choose $\alpha\in(0,1)$ and a finite $\kappa>0$ with $\|\bbW_0\|_2\leq\kappa$, and define
$\overline{\ccalW}_1:=\{\bbW\in\ccalW_1:\rho(\bbW)\leq\alpha,\ \|\bbW\|_2\leq\kappa\}$.
For $\lambda\geq 1$ and $c>0$, $\bbW_0$ remains the unique minimizer of $\bar{L}_c(\bbW,\lambda)$ over $\overline{\ccalW}_1$, since $\bbW_0\in\overline{\ccalW}_1$ and $\overline{\ccalW}_1\subseteq\ccalW_1$.
The bounds can be relaxed arbitrarily by taking $\alpha$ close to $1$ and $\kappa$ large, without changing the minimizer.
Our convergence analysis also requires the following lemma on the Lipschitz smoothness of $\bar{L}_c(\bbW,\lambda)$.
%
\begin{lemma}\label{lem:lipschitz}
    $\bar{L}_c(\bbW,\lambda)$ is $R$-Lipschitz smooth over $\overline{\ccalW}_1$, where
    \begin{equation}\label{eq:lipschitz}
        R 
        :=
        \!
        \big(
            1 + \lambda + cd - cd\log(1-\alpha)
        \big)\tau^2,
        ~~
        \tau 
        \!:=\!
        \left\{\begin{array}{lc}
            \frac{1}{1-\kappa},
            & \!\!\! \kappa < 1 \\
            \frac{(1 + \kappa)^{d-1}}{ (1-\alpha)^{d} },
            & \!\!\! \kappa \geq 1
        \end{array}\right.. \nonumber
    \end{equation}
\end{lemma}
With the above constant $R$, the next result introduces the conditions under which PGD is guaranteed to converge to $\bbW_0$ in the population setting.
%
\begin{theorem}\label{thm:convergence}
    Let the sequence $\{ \bbW_{(\ell)} \}_{\ell=0}^{\infty}$ be given by 
    \begin{equation}\label{eq:pgd}
        \bbW_{(\ell+1)} =
            \Pi_{\overline{\ccalW}_1}\!\Big(
            \bbW_{(\ell)} - \eta \nabla_{\bbW} \bar{L}_c\big( \bbW_{(\ell)}, \lambda\big)
        \Big)
        ~~ \forall \,\ell\geq 0
    \end{equation}
    for $\bbW_{(0)} \in \overline{\ccalW}_1$, $c > 0$, and $\lambda \geq 1$, where $0 < \eta < \frac{1}{R}$ and the operator $\Pi_{\ccalC}(\cdot)$ is an Euclidean projection onto $\ccalC$.
    There exist constants $\epsilon, \delta > 0$ such that, if $\| \bbW_{(\ell^*)} - \bbW_0 \|_F < \epsilon$ and $\bar{L}_c\big(\bbW_{(\ell^*)}, \lambda\big) < \bar{L}_c(\bbW_0,\lambda) + \delta$ for some $\ell^* \in \naturals$, 
    then the sequence $\{ \bbW_{(\ell)} \}_{\ell=0}^{\infty}$ has finite length and converges to $\bbW_0$.
\end{theorem}
The proof of Theorem~\ref{thm:convergence} is given in Appendix~\ref{app:proof_convergence}.
The result establishes local convergence through the favorable combination of $\bar{L}_c$ and $\overline{\ccalW}_1$.
Specifically, the smoothness of $\bar{L}_c(\bbW,\lambda)$ ensures that iterations reduce the objective for sufficiently small stepsizes.
Moreover, the compactness and semialgebraicity of $\overline{\ccalW}_1$, together with the analytic structure of $\bar{L}_c$, support the Kurdyka-{\L}ojasiewicz (KL) argument, ensuring convergence of~\eqref{eq:pgd} to $\bbW_0$.
Theorem~\ref{thm:convergence} thus provides a local justification for NOMAD's inner optimization, guaranteeing convergence to $\bbW_0$ for the analyzed projected-gradient scheme under the stated conditions.
Intuitively, NOMAD's outer iterations may bring the estimate into this local convergence region, although this behavior is not guaranteed.
Nevertheless, the upcoming numerical results demonstrate strong empirical recovery performance.

\section{Numerical Evaluation}\label{sec:numerical_evaluation}
We perform a comprehensive numerical evaluation of NOMAD in pragmatic finite-sample settings. We first consider controlled synthetic experiments, where both the ground-truth DAG and the data-generating process are known, followed by a real-world protein-signaling dataset.
Additional experiments, implementation details, and hyperparameter settings are available in a GitHub repository\footnote{\url{https://github.com/reysam93/cvx_dag_learning}}.

\noindent \textbf{Baselines.} We compare NOMAD against several representative DAG-learning baselines.
Among continuous-optimization approaches, we consider NOTEARS~\cite{zheng2018dags}, DAGMA~\cite{bello2022dagma}, NoFears~\cite{wei2020dags}, GOLEM~\cite{ng2020role}, and CoLiDE~\cite{saboksayr2023colide}. For GOLEM and CoLiDE, we include both the equal noise variance (EV) and non-equal variance (NV) variants. We also consider SortNRegress~\cite{reisach2021beware}, an order-based scheme.
Finally, we introduce NonDAGMA, which projects DAGMA's iterates onto the non-negative orthant to isolate the benefit of our acyclicity characterization beyond restricting the parameter domain.

\noindent \textbf{Metrics.}
We compare the methods in terms of edge-weight estimation and structural accuracy. We measure the former using the normalized squared Frobenius error
\begin{equation}
    \operatorname{nerr}(\hbW,\bbW_0) :=\|\hbW-\bbW_0\|_F^2 / \|\bbW_0\|_F^2,
\end{equation}
where $\hbW$ and $\bbW_0$ denote the estimated and ground-truth weighted adjacency matrices, respectively.
We gauge structural accuracy using the structural Hamming distance (SHD), i.e., the minimum number of edge additions, deletions, and reversals required to transform $\hbW$ into $\bbW_0$.
Depending on the test case, we complement SHD with the true positive rate (TPR), false discovery rate (FDR), $F_1$ score, and normalized structural intervention distance (SID)~\cite{peters2015sid}.
For these metrics, we binarize $\hbW$ by retaining edges with absolute weight above $0.3$.

\subsection{Synthetic Experiments}\label{ssec:synthetic}

We begin our evaluation with synthetic data, which provide a controlled setting for assessing the performance of the methods and gaining insight into their behavior.
We first compare all approaches in the nominal setting described next, then examine three complementary effects: (i) sample size; (ii) exogenous noise variance; and (iii) DAG size and family.
Unless otherwise stated, we generate data from Erd\H{o}s--R\'enyi (ER) DAGs with $d=100$ nodes, an average degree of $4$, and non-negative edge weights drawn uniformly from $[0.5, 1]$. For each DAG, we draw $n=1000$ samples from the linear SEM in \eqref{eq:sem} with standard Gaussian exogenous noise. 
This identifiable setting is widely used in the literature and considered among the hardest~\cite{bello2022dagma,saboksayr2023colide,ng2020role}. When plotting the metrics, each curve depicts the mean and standard deviation over 50 independent realizations.

\begin{figure*}[!t]
    \centering
    \begin{subfigure}[t]{0.3\textwidth}
        \centering
        \begin{tikzpicture}[baseline,scale=1]

\pgfplotstableread{data/samples_err_mean.csv}\errtable
\pgfplotstableread{data/samples_err_std.csv}\stdtable

\pgfplotstablecreatecol[create col/copy column from table={\stdtable}{NOMAD-adam}]{NOMAD-adam-std}{\errtable}
\pgfplotstablecreatecol[create col/copy column from table={\stdtable}{CoLiDE-Fix}]{CoLiDE-Fix-std}{\errtable}
\pgfplotstablecreatecol[create col/copy column from table={\stdtable}{DAGMA}]{DAGMA-std}{\errtable}
\pgfplotstablecreatecol[create col/copy column from table={\stdtable}{NoTears}]{NoTears-std}{\errtable}
\pgfplotstablecreatecol[create col/copy column from table={\stdtable}{GOLEM-EV-Fix}]{GOLEM-EV-Fix-std}{\errtable}
\pgfplotstablecreatecol[create col/copy column from table={\stdtable}{NonDAGMA}]{NonDAGMA-std}{\errtable}

\pgfmathsetmacro{\opacity}{0.3}
\pgfmathsetmacro{\contourop}{0.25}

\begin{loglogaxis}[
    xlabel={(a) Number of samples $n$},
    xmin=50,
    xmax=10000,
    xtick = {100, 1000, 10000},
    xticklabels = {$10^2$, $10^3$, $10^4$},
    ylabel={$nerr(\hbW, \bbW_0)$},
    ymin = 3e-4,
    ymax = 1,
    grid style=densely dashed,
    grid=both,
    legend style={
        at={(0.5, 1.05)},
        anchor=south},
    legend columns=3,
    width=\linewidth,
    height=150,
    ]

    \addplot [blue!80!white, name path = nomad-bot, opacity=\contourop, forget plot] table [x=xaxis, y expr=\thisrow{NOMAD-adam}-\thisrow{NOMAD-adam-std}] \errtable;
    \addplot [blue!70!white, name path = nomad-top, opacity=\contourop, forget plot] table [x=xaxis, y expr=\thisrow{NOMAD-adam}+\thisrow{NOMAD-adam-std}] \errtable;
    \addplot[blue!70!white, fill opacity=\opacity, forget plot] fill between[of=nomad-bot and nomad-top];
    \addplot[blue, solid, mark=o] table [x=xaxis, y=NOMAD-adam] {\errtable};

    \addplot [green!70!black, name path = colide-bot, opacity=\contourop, forget plot] table [x=xaxis, y expr=\thisrow{CoLiDE-Fix}-\thisrow{CoLiDE-Fix-std}] \errtable;
    \addplot [green!70!black, name path = colide-top, opacity=\contourop, forget plot] table [x=xaxis, y expr=\thisrow{CoLiDE-Fix}+\thisrow{CoLiDE-Fix-std}] \errtable;
    \addplot[green!70!black, fill opacity=\opacity, forget plot] fill between[of=colide-bot and colide-top];
    \addplot[green!65!black, solid, mark=x] table [x=xaxis, y=CoLiDE-Fix] {\errtable};

    \addplot [red!80!white, name path = dagma-bot, opacity=\contourop, forget plot] table [x=xaxis, y expr=\thisrow{DAGMA}-\thisrow{DAGMA-std}] \errtable;
    \addplot [red!70!white, name path = dagma-top, opacity=\contourop, forget plot] table [x=xaxis, y expr=\thisrow{DAGMA}+\thisrow{DAGMA-std}] \errtable;
    \addplot[red!70!white, fill opacity=\opacity, forget plot] fill between[of=dagma-bot and dagma-top];
    \addplot[red, solid, mark=star] table [x=xaxis, y=DAGMA] {\errtable};

    \addplot [gray!90!white, name path = notears-bot, opacity=\contourop, forget plot] table [x=xaxis, y expr=\thisrow{NoTears}-\thisrow{NoTears-std}] \errtable;
    \addplot [gray!80!white, name path = notears-top, opacity=\contourop, forget plot] table [x=xaxis, y expr=\thisrow{NoTears}+\thisrow{NoTears-std}] \errtable;
    \addplot[gray!80!white, fill opacity=\opacity, forget plot] fill between[of=notears-bot and notears-top];
    \addplot[gray!90!black, solid, mark=Mercedes star] table [x=xaxis, y=NoTears] {\errtable};

    \addplot [purple!75!white, name path = golem-bot, opacity=\contourop, forget plot] table [x=xaxis, y expr=\thisrow{GOLEM-EV-Fix}-\thisrow{GOLEM-EV-Fix-std}] \errtable;
    \addplot [purple!65!white, name path = golem-top, opacity=\contourop, forget plot] table [x=xaxis, y expr=\thisrow{GOLEM-EV-Fix}+\thisrow{GOLEM-EV-Fix-std}] \errtable;
    \addplot[purple!65!white, fill opacity=\opacity, forget plot] fill between[of=golem-bot and golem-top];
    \addplot[purple!80!black, solid, mark=diamond*] table [x=xaxis, y=GOLEM-EV-Fix] {\errtable};

    \addplot [teal!80!white, name path = nondagma-bot, opacity=\contourop, forget plot] table [x=xaxis, y expr=\thisrow{NonDAGMA}-\thisrow{NonDAGMA-std}] \errtable;
    \addplot [teal!70!white, name path = nondagma-top, opacity=\contourop, forget plot] table [x=xaxis, y expr=\thisrow{NonDAGMA}+\thisrow{NonDAGMA-std}] \errtable;
    \addplot[teal!70!white, fill opacity=\opacity, forget plot] fill between[of=nondagma-bot and nondagma-top];
    \addplot[teal!70!black, solid, mark=triangle*] table [x=xaxis, y=NonDAGMA] {\errtable};

    \legend{NOMAD, CoLiDE, DAGMA, NoTears, GOLEM, NonDAGMA}
\end{loglogaxis}
\end{tikzpicture}
    \end{subfigure}
    \begin{subfigure}[t]{0.3\textwidth}
        \centering
        \begin{tikzpicture}[baseline,scale=1]

\pgfplotstableread{data/vars_hom_err_mean.csv}\errtable
\pgfplotstableread{data/vars_hom_err_std.csv}\stdtable

\pgfplotstablecreatecol[create col/copy column from table={\stdtable}{NOMAD-adam}]{NOMAD-adam-std}{\errtable}
\pgfplotstablecreatecol[create col/copy column from table={\stdtable}{DAGMA}]{DAGMA-std}{\errtable}
\pgfplotstablecreatecol[create col/copy column from table={\stdtable}{NoFears}]{NoFears-std}{\errtable}
\pgfplotstablecreatecol[create col/copy column from table={\stdtable}{NonDAGMA}]{NonDAGMA-std}{\errtable}
\pgfplotstablecreatecol[create col/copy column from table={\stdtable}{CoLiDE-EV}]{CoLiDE-EV-std}{\errtable}
\pgfplotstablecreatecol[create col/copy column from table={\stdtable}{GOLEM-EV}]{GOLEM-EV-std}{\errtable}

\pgfmathsetmacro{\opacity}{0.3}
\pgfmathsetmacro{\contourop}{0.25}

\begin{semilogyaxis}[
    xlabel={(b) Exogenous noise variance $\sigma^2$},
    xmin=1,
    xmax=30,
    xtick = {1, 5, 10, 20, 30},
    xticklabels = {1, 5, 10, 20, 30},
    ylabel={$nerr(\hbW, \bbW_0)$},
    ymin = 3e-3,
    ymax = .35,
    grid style=densely dashed,
    grid=both,
    legend style={
        font=\scriptsize,
        at={(0.5, 1.03)},
        anchor=south},
    legend columns=3,
    width=\linewidth,
    height=150,
    ]

    \addplot [blue!80!white, name path = nomad-bot, opacity=\contourop, forget plot] table [x=xaxis, y expr={max(\thisrow{NOMAD-adam}-\thisrow{NOMAD-adam-std},1e-6)}] \errtable;
    \addplot [blue!70!white, name path = nomad-top, opacity=\contourop, forget plot] table [x=xaxis, y expr=\thisrow{NOMAD-adam}+\thisrow{NOMAD-adam-std}] \errtable;
    \addplot[blue!70!white, fill opacity=\opacity, forget plot] fill between[of=nomad-bot and nomad-top];
    \addplot[blue, solid, mark=o] table [x=xaxis, y=NOMAD-adam] {\errtable};

    \addplot [red!80!white, name path = dagma-bot, opacity=\contourop, forget plot] table [x=xaxis, y expr={max(\thisrow{DAGMA}-\thisrow{DAGMA-std},1e-6)}] \errtable;
    \addplot [red!70!white, name path = dagma-top, opacity=\contourop, forget plot] table [x=xaxis, y expr=\thisrow{DAGMA}+\thisrow{DAGMA-std}] \errtable;
    \addplot[red!70!white, fill opacity=\opacity, forget plot] fill between[of=dagma-bot and dagma-top];
    \addplot[red, solid, mark=star] table [x=xaxis, y=DAGMA] {\errtable};

    \addplot [gray!90!white, name path = nofears-bot, opacity=\contourop, forget plot] table [x=xaxis, y expr={max(\thisrow{NoFears}-\thisrow{NoFears-std},1e-6)}] \errtable;
    \addplot [gray!80!white, name path = nofears-top, opacity=\contourop, forget plot] table [x=xaxis, y expr=\thisrow{NoFears}+\thisrow{NoFears-std}] \errtable;
    \addplot[gray!80!white, fill opacity=\opacity, forget plot] fill between[of=nofears-bot and nofears-top];
    \addplot[gray!90!black, solid, mark=Mercedes star] table [x=xaxis, y=NoFears] {\errtable};

    \addplot [teal!80!white, name path = nondagma-bot, opacity=\contourop, forget plot] table [x=xaxis, y expr={max(\thisrow{NonDAGMA}-\thisrow{NonDAGMA-std},1e-6)}] \errtable;
    \addplot [teal!70!white, name path = nondagma-top, opacity=\contourop, forget plot] table [x=xaxis, y expr=\thisrow{NonDAGMA}+\thisrow{NonDAGMA-std}] \errtable;
    \addplot[teal!70!white, fill opacity=\opacity, forget plot] fill between[of=nondagma-bot and nondagma-top];
    \addplot[teal!70!black, solid, mark=triangle*] table [x=xaxis, y=NonDAGMA] {\errtable};

    \addplot [green!70!black, name path = colide-bot, opacity=\contourop, forget plot] table [x=xaxis, y expr={max(\thisrow{CoLiDE-EV}-\thisrow{CoLiDE-EV-std},1e-6)}] \errtable;
    \addplot [green!70!black, name path = colide-top, opacity=\contourop, forget plot] table [x=xaxis, y expr=\thisrow{CoLiDE-EV}+\thisrow{CoLiDE-EV-std}] \errtable;
    \addplot[green!70!black, fill opacity=\opacity, forget plot] fill between[of=colide-bot and colide-top];
    \addplot[green!65!black, solid, mark=x] table [x=xaxis, y=CoLiDE-EV] {\errtable};

    \addplot [purple!75!white, name path = golem-bot, opacity=\contourop, forget plot] table [x=xaxis, y expr={max(\thisrow{GOLEM-EV}-\thisrow{GOLEM-EV-std},1e-6)}] \errtable;
    \addplot [purple!65!white, name path = golem-top, opacity=\contourop, forget plot] table [x=xaxis, y expr=\thisrow{GOLEM-EV}+\thisrow{GOLEM-EV-std}] \errtable;
    \addplot[purple!65!white, fill opacity=\opacity, forget plot] fill between[of=golem-bot and golem-top];
    \addplot[purple!80!black, solid, mark=diamond*] table [x=xaxis, y=GOLEM-EV] {\errtable};

    \legend{NOMAD, DAGMA, NoFears, NonDAGMA, CoLiDE, GOLEM}
\end{semilogyaxis}
\end{tikzpicture}
    \end{subfigure}
    \begin{subfigure}[t]{0.3\textwidth}
        \centering
        \begin{tikzpicture}[baseline,scale=1]

\pgfplotstableread{data/size_er_shd_mean.csv}\ertable
\pgfplotstableread{data/size_sf_shd_mean.csv}\sftable
\pgfplotstableread{data/size_er_shd_std.csv}\erstdtable
\pgfplotstableread{data/size_sf_shd_std.csv}\sfstdtable

\pgfplotstablecreatecol[create col/copy column from table={\erstdtable}{NOMAD-adam}]{NOMAD-adam-std}{\ertable}
\pgfplotstablecreatecol[create col/copy column from table={\erstdtable}{DAGMA}]{DAGMA-std}{\ertable}
\pgfplotstablecreatecol[create col/copy column from table={\erstdtable}{SortNRegress}]{SortNRegress-std}{\ertable}
\pgfplotstablecreatecol[create col/copy column from table={\sfstdtable}{NOMAD-adam}]{NOMAD-adam-std}{\sftable}
\pgfplotstablecreatecol[create col/copy column from table={\sfstdtable}{DAGMA}]{DAGMA-std}{\sftable}
\pgfplotstablecreatecol[create col/copy column from table={\sfstdtable}{SortNRegress}]{SortNRegress-std}{\sftable}

\pgfmathsetmacro{\opacity}{0.3}
\pgfmathsetmacro{\contourop}{0.25}

\begin{semilogxaxis}[
    xlabel={(c) Number of nodes $d$},
    xmin=50,
    xmax=1000,
    xtick = {50, 100, 250, 500, 1000},
    xticklabels = {50, 100, 250, 500, 1000},
    ylabel={Normalized SHD},
    ymin = -0.01,
    ymax = 3.8,
    grid style=densely dashed,
    grid=both,
    legend style={
        at={(0.5, 1.01)},
        anchor=south},
    legend columns=2,
    width=\linewidth,
    height=150,
    ]

    \addplot[blue!80!black, name path=nomad-er-bot, opacity=\contourop, forget plot] table [x=xaxis, y expr={max(\thisrow{NOMAD-adam}-\thisrow{NOMAD-adam-std},0)}] \ertable;
    \addplot[blue!80!black, name path=nomad-er-top, opacity=\contourop, forget plot] table [x=xaxis, y expr=\thisrow{NOMAD-adam}+\thisrow{NOMAD-adam-std}] \ertable;
    \addplot[blue!80!black, fill opacity=\opacity, forget plot] fill between[of=nomad-er-bot and nomad-er-top];
    \addplot[blue!80!black, solid, mark=o] table [x=xaxis, y=NOMAD-adam] {\ertable};

    \addplot[blue!60!white, name path=nomad-sf-bot, opacity=\contourop, forget plot] table [x=xaxis, y expr={max(\thisrow{NOMAD-adam}-\thisrow{NOMAD-adam-std},0)}] \sftable;
    \addplot[blue!60!white, name path=nomad-sf-top, opacity=\contourop, forget plot] table [x=xaxis, y expr=\thisrow{NOMAD-adam}+\thisrow{NOMAD-adam-std}] \sftable;
    \addplot[blue!60!white, fill opacity=\opacity, forget plot] fill between[of=nomad-sf-bot and nomad-sf-top];
    \addplot[blue!60!white, densely dotted, mark=o] table [x=xaxis, y=NOMAD-adam] {\sftable};

    \addplot[red!80!black, name path=dagma-er-bot, opacity=\contourop, forget plot] table [x=xaxis, y expr={max(\thisrow{DAGMA}-\thisrow{DAGMA-std},0)}] \ertable;
    \addplot[red!80!black, name path=dagma-er-top, opacity=\contourop, forget plot] table [x=xaxis, y expr=\thisrow{DAGMA}+\thisrow{DAGMA-std}] \ertable;
    \addplot[red!80!black, fill opacity=\opacity, forget plot] fill between[of=dagma-er-bot and dagma-er-top];
    \addplot[red!80!black, solid, mark=star] table [x=xaxis, y=DAGMA] {\ertable};

    \addplot[red!60!white, name path=dagma-sf-bot, opacity=\contourop, forget plot] table [x=xaxis, y expr={max(\thisrow{DAGMA}-\thisrow{DAGMA-std},0)}] \sftable;
    \addplot[red!60!white, name path=dagma-sf-top, opacity=\contourop, forget plot] table [x=xaxis, y expr=\thisrow{DAGMA}+\thisrow{DAGMA-std}] \sftable;
    \addplot[red!60!white, fill opacity=\opacity, forget plot] fill between[of=dagma-sf-bot and dagma-sf-top];
    \addplot[red!60!white, densely dotted, mark=star] table [x=xaxis, y=DAGMA] {\sftable};

    \addplot[purple!80!black, name path=sort-er-bot, opacity=\contourop, forget plot] table [x=xaxis, y expr={max(\thisrow{SortNRegress}-\thisrow{SortNRegress-std},0)}] \ertable;
    \addplot[purple!80!black, name path=sort-er-top, opacity=\contourop, forget plot] table [x=xaxis, y expr=\thisrow{SortNRegress}+\thisrow{SortNRegress-std}] \ertable;
    \addplot[purple!80!black, fill opacity=\opacity, forget plot] fill between[of=sort-er-bot and sort-er-top];
    \addplot[purple!80!black, solid, mark=square*] table [x=xaxis, y=SortNRegress] {\ertable};

    \addplot[purple!60!white, name path=sort-sf-bot, opacity=\contourop, forget plot] table [x=xaxis, y expr={max(\thisrow{SortNRegress}-\thisrow{SortNRegress-std},0)}] \sftable;
    \addplot[purple!60!white, name path=sort-sf-top, opacity=\contourop, forget plot] table [x=xaxis, y expr=\thisrow{SortNRegress}+\thisrow{SortNRegress-std}] \sftable;
    \addplot[purple!60!white, fill opacity=\opacity, forget plot] fill between[of=sort-sf-bot and sort-sf-top];
    \addplot[purple!60!white, densely dotted, mark=square*] table [x=xaxis, y=SortNRegress] {\sftable};

    \legend{NOMAD-ER, NOMAD-SF, DAGMA-ER, DAGMA-SF, SortNReg.-ER, SortNReg.-SF}
\end{semilogxaxis}
\end{tikzpicture}
    \end{subfigure}
    \vspace*{-2mm} 
    \caption{Synthetic evaluation of the proposed DAG-learning method.
    (a) Normalized squared Frobenius error as the number of samples increases.
    (b) Normalized squared Frobenius error as the variance of the exogenous noise increases.
    (c) Normalized structural Hamming distance (SHD) as the graph size increases, for Erd\H{o}s--R\'enyi (ER) and scale-free (SF) DAGs.}
    \label{fig:numerical_evaluation}
    \vspace*{-5mm}
\end{figure*}

\noindent \textbf{Test case 1 - Baseline comparison.}
Our first experiment compares the different DAG-learning methods in the reference setting described earlier.
We report two implementations of NOMAD, which solve the inner problem \eqref{eq:al_w_update} using either Adam~\cite{kingma2015adam} or FISTA~\cite{beck2009fast}.
As shown in \cref{tab:synthetic_results}, both variants achieve the best structural and estimation accuracy, with the lowest estimation errors and near-perfect mean SHD and $F_1$ scores.
The small standard deviations further indicate that this high recovery accuracy is consistent across realizations.
Both NOMAD variants are also faster than the continuous-optimization baselines considered, although the order-based SortNRegress method remains substantially faster at the expense of lower estimation accuracy.
Moreover, NonDAGMA markedly improves upon DAGMA across all accuracy metrics, highlighting the benefit of enforcing non-negativity.
NOMAD further outperforms NonDAGMA, providing empirical evidence that exploiting non-negativity in the acyclicity characterization is beneficial beyond restricting the parameter domain alone. 
This gap becomes more evident in subsequent experiments.

\noindent \textbf{Test case 2 - Sample complexity.}
To assess the effect of sample size on estimation performance, \cref{fig:numerical_evaluation}(a) reports the estimation error as the number of samples increases.
The error of competing methods eventually plateaus despite the availability of additional samples. In contrast, NOMAD consistently attains the lowest estimation error, which steadily approaches zero as $n$ grows.
When comparing NOMAD with NonDAGMA, we observe that the latter remains competitive under limited data, but its error begins to saturate beyond $1000$ samples, thereby widening the gap between the two methods.
This behavior is consistent with the benign population landscape established in \cref{sec:optimization_landscape} through the novel acyclicity characterization.
Taken together, our theoretical and empirical results indicate that NOMAD asymptotically recovers the ground-truth DAG, i.e., the global optimum of the population problem, whereas competing methods appear to converge to suboptimal solutions.

\noindent\textbf{Test case 3 - Noise variance.} 
Next, we study the performance of the different methods under varying exogenous noise levels in the homoscedastic setting, where $\bbSigma_{\bbz}=\sigma^2\bbI$. \cref{fig:numerical_evaluation}(b) reports the estimation error as $\sigma^2$ increases.
As expected, CoLiDE and GOLEM, which jointly estimate the DAG adjacency matrix and the exogenous noise variance, exhibit a fairly stable performance throughout the range.
Although NOMAD does not explicitly adapt to changes in $\sigma^2$, its error increases only slightly at first and then stabilizes around $0.01$, remaining the lowest for all variance values.
This behavior is consistent with the fact that the problem remains identifiable throughout the tested range, allowing NOMAD to approach the ground-truth DAG given sufficiently many samples.
NonDAGMA again markedly outperforms DAGMA, particularly at low noise levels, but its error increases more sharply than that of NOMAD with a considerable increase in standard deviation.
Once again, these results corroborate that while accounting for non-negativity yields substantial gains, they fall short of those afforded by NOMAD's simpler acyclicity characterization.
Additional experiments with non-Gaussian noise (e.g., Laplace and exponential distributions), available in the GitHub repository, yield the same qualitative conclusions.

\noindent \textbf{Test case 4 - Graph size and family.} We conclude the synthetic evaluation by studying how graph size and family affect structural recovery. To that end, \cref{fig:numerical_evaluation}(c) reports the SHD normalized by the number of nodes as $d$ increases for ER and scale-free (SF) DAGs.
For ER DAGs, NOMAD yields almost exact structural recovery throughout, with normalized SHD close to zero, whereas DAGMA and SortNRegress exhibit appreciable structural errors.
NOMAD and DAGMA have similar runtime scaling with $d$.
For reference, NOMAD takes approximately 0.77, 2.7, and 274 minutes per instance with $50$, $100$, and $1000$ nodes, respectively. Interestingly, SF DAGs present a more challenging setting for NOMAD and DAGMA, whose normalized SHD increases markedly at larger graph sizes. Nevertheless, NOMAD attains the lowest normalized SHD up to $d=500$ and consistently outperforms DAGMA. At $d=1000$, however, SortNRegress achieves the lowest error.

\subsection{Sachs Protein-Signaling Benchmark}\label{ssec:sachs_experiment}

Here we test NOMAD on the Sachs protein-signaling dataset~\cite{sachs2005causal}, widely regarded as the real-data benchmark of choice for evaluating methods for DAG structure learning and causal discovery.
The dataset contains single-cell measurements of phosphorylated proteins and phospholipids under experimental perturbations, together with a biologically supported reference network shown in \cref{fig:sachs_graphs}(a).
Its $d=11$ nodes represent measured signaling components, and its $17$ directed edges encode binary regulatory interactions. For DAG estimation, we use the $n=853$ available observations after centering.

The performance of the DAG-learning methods considered is reported in \cref{tab:sachs_results}, which shows that NOMAD attains the best results across all metrics.
In particular, NOMAD achieves a modest improvement in SHD relative to the second-best performer, CoLiDE-NV.
More importantly, NOMAD attains an FDR of zero, indicating that every recovered edge corresponds to a true interaction with the correct orientation.
This yields an $F_1$ score of $0.583$, substantially above those of the competing methods.
In terms of edge-weight recovery, the proposed method also reduces the estimation error considerably.
Nevertheless, this error remains relatively high, highlighting that Sachs is a challenging benchmark in which accurate edge-weight estimation is non-trivial.
For a qualitative comparison, the DAGs estimated by CoLiDE-NV and NOMAD are shown in \cref{fig:sachs_graphs}(b) and \cref{fig:sachs_graphs}(c), respectively.
CoLiDE-NV recovers much of the underlying skeleton but incorrectly orients several edges, whereas NOMAD provides a more faithful reconstruction at the expense of missing some connections.
Overall, these results highlight how exploiting the non-negativity assumption can improve DAG estimation even in real-world settings.

\begin{table}[!t]
    \centering
    \caption{Results on the Sachs dataset. Best and second-best values in bold and underlined, as in \cref{tab:synthetic_results}.}
    \label{tab:sachs_results}
    \scriptsize
    \setlength{\tabcolsep}{3.5pt}
    \resizebox{\columnwidth}{!}{%
    \begin{tabular}{lcccccc}
        \toprule
        Method & SHD $\downarrow$ & TPR $\uparrow$ & FDR $\downarrow$ & $F_1$ score $\uparrow$ & SID $\downarrow$ & nerr $\downarrow$ \\
        \midrule
        NOMAD & \textbf{10} & \textbf{0.412} & \textbf{0.000} & \textbf{0.583} & \textbf{1.545} & \textbf{0.729} \\
        \midrule
        DAGMA & 13 & 0.235 & \underline{0.333} & 0.348 & \underline{2.000} & 1.620 \\
        NonDAGMA & 14 & 0.176 & 0.500 & 0.261 & 2.182 & \underline{1.435} \\
        \midrule
        CoLiDE-EV & 13 & \underline{0.353} & 0.571 & 0.387 & --- & 1.932 \\
        CoLiDE-NV & \underline{12} & \underline{0.353} & 0.538 & \underline{0.400} & 2.364 & 1.950 \\
        GOLEM-EV & 18 & \textbf{0.412} & 0.650 & 0.378 & 4.545 & 1.843 \\
        GOLEM-NV & 14 & 0.176 & 0.500 & 0.261 & 2.091 & 1.856 \\
        SortNReg. & 13 & 0.294 & 0.615 & 0.333 & 2.455 & 1.996 \\
        \bottomrule
    \end{tabular}%
    }
    \vspace*{-3mm}
\end{table}

\begin{figure*}[!t]
    \centering
    \begin{subfigure}[t]{0.22\textwidth}
        \centering
        \includegraphics[width=\linewidth]{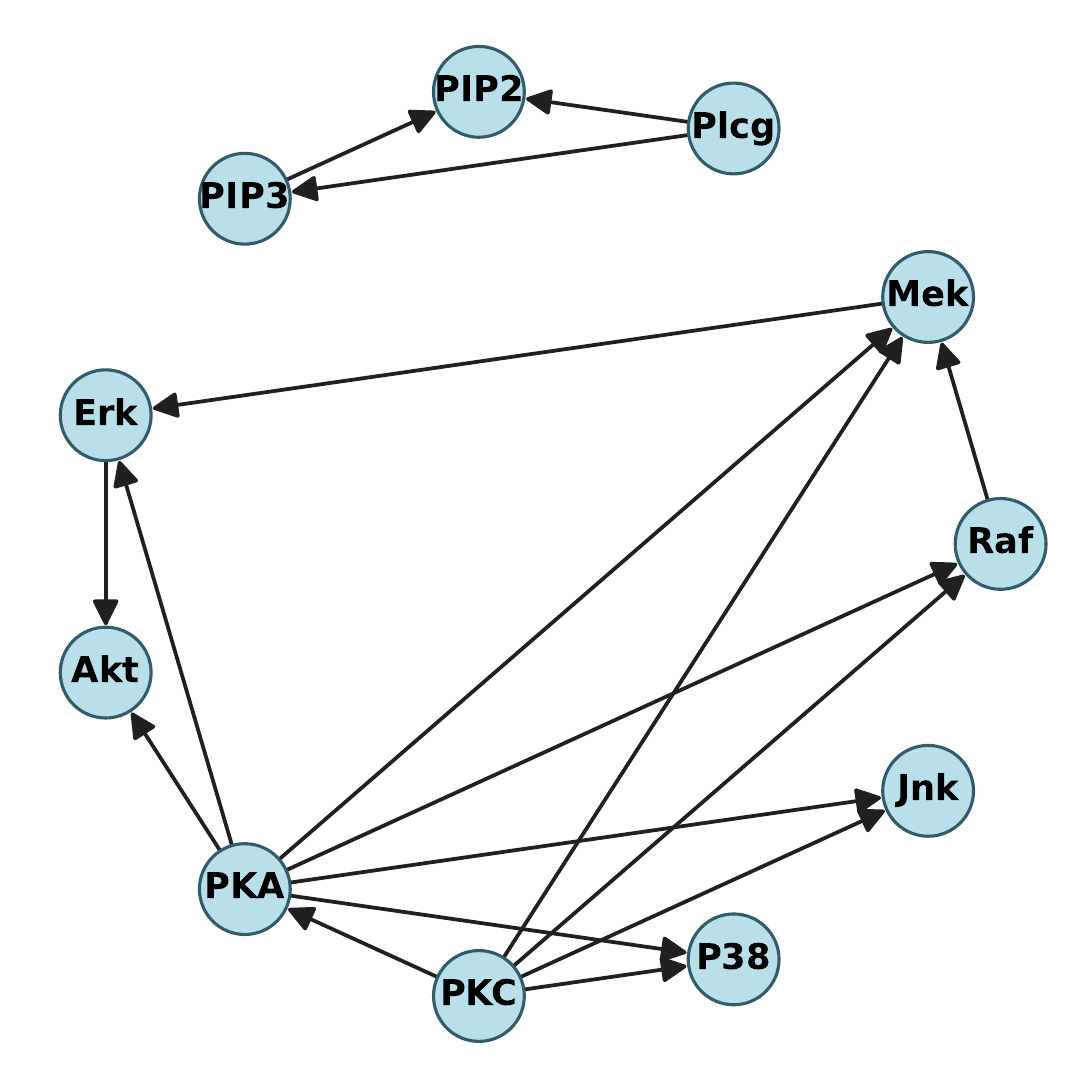}
        \caption{Reference DAG}
    \end{subfigure}
    \begin{subfigure}[t]{0.22\textwidth}
        \centering
        \includegraphics[width=\linewidth]{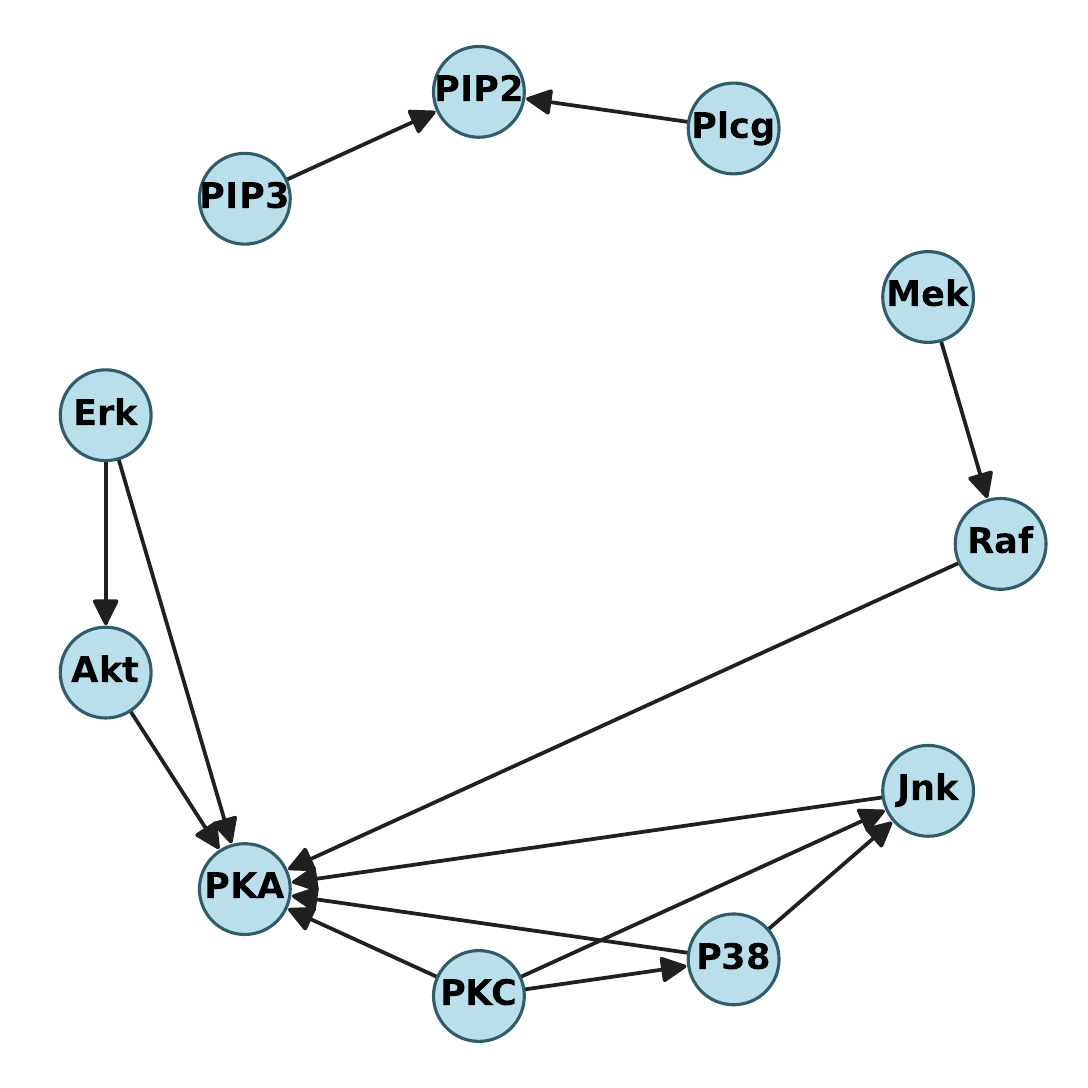}
        \caption{CoLiDE-NV estimate}
    \end{subfigure}
    \begin{subfigure}[t]{0.22\textwidth}
        \centering
        \includegraphics[width=\linewidth]{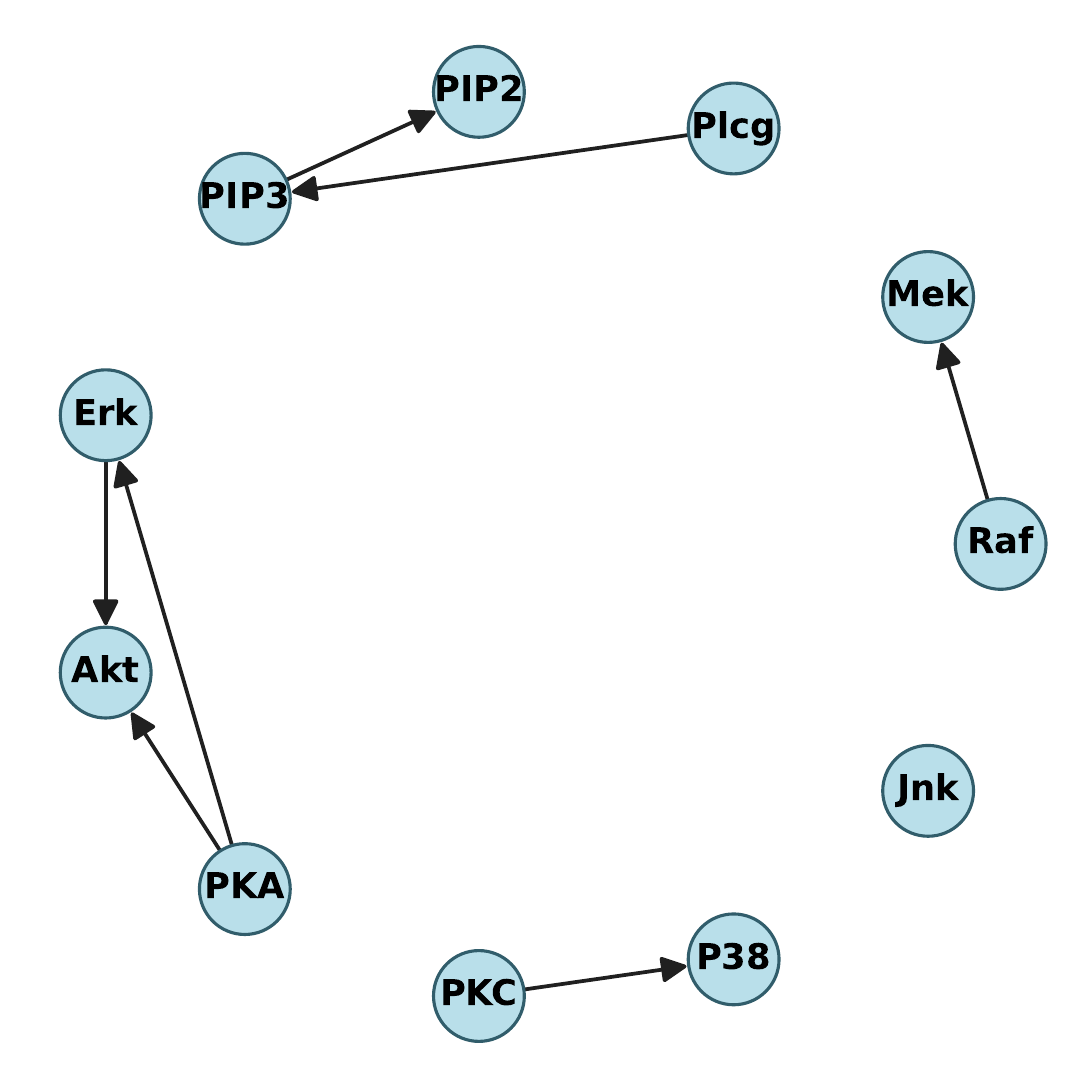}
        \caption{NOMAD estimate}
    \end{subfigure}
    \caption{Qualitative comparison on the Sachs benchmark. NOMAD is closer to the reference DAG than CoLiDE-NV, with no false discoveries, in agreement with the SHD, $F_1$ score, and FDR values reported in \cref{tab:sachs_results}.}
    \label{fig:sachs_graphs}
    \vspace*{-3mm}
\end{figure*}


\section{Conclusions, Limitations, and Future Work}\label{sec:conclusions}
This paper addressed the problem of learning the DAG structure from observations generated by a linear SEM under the assumption of non-negative edge weights. By ruling out cancellations among closed-walk contributions, non-negativity enabled a simple, smooth characterization of acyclicity that avoids Hadamard products. Building on this characterization, we introduced NOMAD, a tractable MoM algorithm for estimating non-negative weighted DAGs. Under the stated population SEM and identifiability conditions, we showed that the true DAG is the unique global minimizer of the population augmented Lagrangian, that there are no spurious interior stationary points, and that it is the only acyclic KKT point. We also established local convergence of NOMAD's projected-gradient iterations under suitable conditions. Experiments on controlled synthetic data and the Sachs protein-signaling benchmark supported the theoretical analysis and showed that exploiting non-negativity can improve both edge-weight estimation and structural recovery relative to the considered DAG-learning baselines.
Overall, these results suggest that incorporating meaningful sign information can yield a favorable optimization landscape for DAG learning, providing a promising route for structure learning in applications where interactions are known to be non-inhibitory.

Notwithstanding the promising results of the proposed framework, we acknowledge several limitations. Our formulation assumes a linear SEM with i.i.d. observations, and the landscape and local convergence guarantees are established at the population level.
These shortcomings suggest several directions for future work. Natural extensions include nonlinear SEMs parameterized by neural networks or kernel expansions, as well as non-negative structural vector autoregressive models for dependent time-series observations. Establishing finite-sample estimation error bounds and high-dimensional consistency would complement the population-level analysis. Developing online algorithms could further enable learning from data streams and tracking dynamic non-negative DAGs.

\appendices
\section{Proof of \cref{thm:population_unique_minimizer}}\label{app:proof_population_unique_minimizer}

\begin{proof} We first state a preliminary supporting lower bound that reduces the analysis of the augmented Lagrangian to a scalar function of $h(\bbW)$.
Moreover, the following definitions 
\begin{equation}\label{eq:proof_defs}
    \bbM:=\bbI-\bbW,\qquad
    \bbM_0:=\bbI-\bbW_0,\qquad
    \bbC:=\bbM_0^{-1}\bbM
\end{equation}
will be used in the different proofs.

\begin{lemma}\label{lem:population_lower_bound}
Under Assumption~\ref{ass:population_sem}, let $c>0$ and $\lambda\geq 1$. Then, for every $\bbW\in\ccalW_1$, $\bar{L}_c(\bbW,\lambda) \geq \phi_{\lambda,c}(h(\bbW))$, where
\begin{equation}\label{eq:phi_def}
    \phi_{\lambda,c}(t)
    :=
    \frac{d}{2} e^{-2t/d}
    +
    \lambda t
    +
    \frac{c}{2}t^2,
    \qquad t\geq 0.
\end{equation}
Moreover, $\phi_{\lambda,c}(t)\geq d/2$ for every $t\geq 0$, with equality if and only if $t=0$.
\end{lemma}

\begin{proof}
The proof proceeds in two steps.
First, we relate the population score to the acyclicity value by rewriting the score via the matrix $\bbC$, which enables relating its determinant to $h(\bbW)$.
Second, after substituting this bound into the augmented Lagrangian, it remains only to analyze a one-dimensional function of $h(\bbW)$.

We begin by rewriting the population score in \eqref{eq:population_score}.
Since the exogenous noise has identity covariance, the SEM implies that the covariance matrix of the data is given by
\begin{equation}
    \bbSigma_{\bbx} =
    (\bbI-\bbW_0)^{-\top} (\bbI-\bbW_0)^{-1} =
    \bbM_0^{-\top}\bbM_0^{-1}.
\end{equation}
Therefore, the score function can be expressed as
\begin{equation}\label{eq:normalized_population_score}
    \bar{F}(\bbW) =
    \frac{1}{2}\tr(\bbM^\top\bbSigma_{\bbx}\bbM) =
    \frac{1}{2}\tr(\bbC^\top\bbC).
\end{equation}

Shifting our attention to the acyclicity function, first recall that $\bbW_0$ is a DAG.
It follows that $\bbW_0$ is a nilpotent matrix, and hence $\det(\bbM_0)=\det(\bbI-\bbW_0)=1$.
Thus,
\begin{equation}\label{eq:detC_detM}
    \det(\bbC)
    =
    \det(\bbM_0^{-1}\bbM)
    =
    \det(\bbM).
\end{equation}
Since $\bbW\in\ccalW_1$, the log-determinant is well defined and
\begin{align}\label{eq:normalized_acyclicity}
    h(\bbW) &= 
    -\log\det(\bbM) 
    = -\log\det(\bbC).
\end{align}
Letting $\sigma_1,\ldots,\sigma_d$ be the singular values of $\bbC$, and noticing that $\det(\bbC)>0$ since $\det(\bbM)>0$ [cf. \eqref{eq:detC_detM}], we have
\begin{equation}\label{eq:det_to_exp}
    \prod_{i=1}^d \sigma_i = \det(\bbC) = e^{-h(\bbW)}.
\end{equation}

Next, to relate the score \eqref{eq:population_score} and the acyclicity function \eqref{eq:nonneg_logdet_constraint}, notice that $\bar{F}(\bbW)  =
    \frac{1}{2}\tr(\bbC^\top\bbC) = \frac{1}{2}\sum_{i=1}^d \sigma_i^2$.
Applying the arithmetic-geometric mean inequality to the sequence of squared singular values $\sigma_1^2,\ldots,\sigma_d^2$ together with \eqref{eq:det_to_exp} yields
\begin{equation}\label{eq:am_gm_inequality}
    \bar{F}(\bbW)
    =
    \frac{1}{2}\sum_{i=1}^d \sigma_i^2
    \geq
    \frac{d}{2}
    \left(
    \prod_{i=1}^d \sigma_i^2
    \right)^{1/d}
    =
    \frac{d}{2} e^{-2h(\bbW)/d}.
\end{equation}
Substituting this bound into the expression for $\bar{L}_c$ in \eqref{eq:population_augmented_lagrangian} and using the function $\phi_{\lambda,c}$ defined in \eqref{eq:phi_def} yields
%
\begin{align}
    \bar{L}_c &(\bbW,\lambda)
    = \bar{F}(\bbW) + \lambda h(\bbW) + \frac{c}{2}h^2(\bbW)  \\
    &\geq
    \frac{d}{2} e^{-2h(\bbW)/d}
    +
    \lambda h(\bbW)
    +
    \frac{c}{2}h^2(\bbW) =
    \phi_{\lambda,c}(h(\bbW)), \nonumber
\end{align}
which reduces the augmented-Lagrangian bound to the scalar function $\phi_{\lambda,c}$ evaluated at the acyclicity residual.

It remains to show that this scalar lower bound is never below $\phi_{\lambda,c}(0)=d/2$, the value attained for a DAG (i.e., $h(\bbW)=0$). For $t\geq 0$, the first- and second-order derivatives of $\phi_{\lambda,c}(t)$ are given by
\begin{align}
    &\phi'_{\lambda,c}(t) =
    -e^{-2t/d} + \lambda + ct, \label{eq:derivative_phi} \\
    &\phi''_{\lambda,c}(t) = \frac{2}{d}e^{-2t/d} + c > 0.
\end{align}
Thus $\phi_{\lambda,c}$ is strictly convex and $\phi'_{\lambda,c}$ is strictly increasing.
Since $\lambda\geq 1$, we have $\phi'_{\lambda,c}(0)=\lambda-1\geq 0$.
Hence $\phi_{\lambda,c}$ is minimized at $t=0$ over $[0,\infty)$, and strict convexity gives
\begin{equation}
    \phi_{\lambda,c}(t)
    \geq
    \phi_{\lambda,c}(0)
    =
    \frac{d}{2},
\end{equation}
with equality if and only if $t=0$, concluding the proof.
\end{proof}

We now leverage the result from \cref{lem:population_lower_bound} to first prove that $\bbW_0$ is a global minimizer of $\bar{L}_c(\cdot,\lambda)$ over $\ccalW_1$, and then prove that it is the unique global minimizer.
Since $\bbW_0$ is a DAG, $h(\bbW_0)=0$.
Moreover, from the SEM relation $(\bbI-\bbW_0^\top)\bbx=\bbz$ and $\mbE[\bbz\bbz^\top]=\bbI$, it follows that
\begin{equation}\label{eq:true_dag_score}
    \bar{L}_c(\bbW_0,\lambda) =
    \bar{F}(\bbW_0) =
    \frac{1}{2}\mbE[\|\bbz\|_2^2] =
    \frac{d}{2}.
\end{equation}
For any $\bbW\in\ccalW_1$ and $\lambda \geq 1$, \cref{lem:population_lower_bound} gives
\begin{equation}\label{eq:pop_Lagr_ineqs}
    \bar{L}_c(\bbW,\lambda)
    \geq
    \phi_{\lambda,c}(h(\bbW))
    \geq
    \frac{d}{2}
    =
    \bar{L}_c(\bbW_0,\lambda).
\end{equation}
Therefore, $\bbW_0$ is a global minimizer of $\bar{L}_c(\cdot,\lambda)$ over $\ccalW_1$.

To prove uniqueness, let $\bbW\in\ccalW_1$ be any global minimizer. Then $\bar{L}_c(\bbW,\lambda)=d/2$.
The inequalities in \eqref{eq:pop_Lagr_ineqs} can be tight only if $h(\bbW)=0$.
By \cref{prop:nonneg_logdet}, this means that $\bbW$ is a DAG.
Furthermore, every feasible DAG $\widetilde{\bbW}\in\ccalW_1$ satisfies $h(\widetilde{\bbW})=0$, so $\bar{L}_c(\widetilde{\bbW},\lambda) = \bar{F}(\widetilde{\bbW})$.
Since $\bbW$ is a global minimizer of $\bar{L}_c$, it follows that $\bbW$ also minimizes $\bar{F}$ over the feasible DAG set.
The identifiability in \eqref{eq:population_identifiability} then implies $\bbW=\bbW_0$ and proving that $\bbW_0$ is the unique global minimizer.
\end{proof}

\section{Proof of \cref{thm:population_stationary_points}}\label{app:proof_population_stationary_points}
\begin{proof}
The argument uses the scalar function $\phi_{\lambda,c}$ introduced in \cref{lem:population_lower_bound} to analyze the first-order equation $\nabla_{\bbW}\bar{L}_c(\bbW,\lambda)=\bbzero$.    
The main idea is to express the augmented Lagrangian in terms of the normalized matrix $\bbC=\bbM_0^{-1}\bbM$ [see \eqref{eq:proof_defs}].
In these coordinates, stationarity forces all singular values of $\bbC$ to be equal.
Combining this with the determinant identity encoded by $h(\bbW)$ reduces the matrix stationarity condition to a scalar equation involving $\phi'_{\lambda,c}(h(\bbW))$.

As in the proof of \cref{lem:population_lower_bound} [see \eqref{eq:normalized_population_score} and \eqref{eq:normalized_acyclicity}], leveraging the definitions in \eqref{eq:proof_defs} we can rewrite the score and the acyclicity function as
\begin{equation}\label{eq:score_acylicity_lemma1}
    \bar{F}(\bbW)=\frac{1}{2}\tr(\bbC^\top\bbC),
    \qquad
    h(\bbW)=-\log\det(\bbC).
\end{equation}
Thus, with a slight abuse of notation, the population augmented Lagrangian can be written as a function of the matrix $\bbC$, leading to
\begin{equation}
    \bar{L}_c(\bbC,\lambda) \! = \!
    \frac{1}{2}\tr(\bbC^\top\bbC) \!-\! \lambda \! \log\det(\bbC) \!+\! \frac{c}{2} \! \bigl(\log\det(\bbC)\bigr)^2.
\end{equation}

The affine map $\bbW\mapsto\bbC=\bbM_0^{-1}(\bbI-\bbW)$ is invertible.
Consequently, $\nabla_{\bbW}\bar{L}_c(\bbW,\lambda)=\bbzero$ if and only if $\nabla_{\bbC}\bar{L}_c(\bbC,\lambda)=\bbzero$.
Differentiating with respect to $\bbC$ gives
\begin{align}
    \nabla_{\bbC}\bar{L}_c(\bbC,\lambda)
    &=
    \bbC
    -
    \lambda\bbC^{-\top}
    +
    c\log\det(\bbC)\bbC^{-\top}
    \nonumber\\
    &=
    \bbC
    -
    \bigl(\lambda+c h(\bbW)\bigr)\bbC^{-\top},
\end{align}
where in the last equality we used \eqref{eq:score_acylicity_lemma1}. 
Hence, the stationarity condition $\nabla_{\bbC}\bar{L}_c(\bbC,\lambda) = \bbzero$ amounts to
\begin{equation}\label{eq:stationary_C_equation}
    \bbC = \bigl(\lambda+c h(\bbW)\bigr)\bbC^{-\top}.
\end{equation}

We now turn this matrix equation into a scalar condition.
Multiplying both sides of \eqref{eq:stationary_C_equation} on the right by $\bbC^\top$ yields
\begin{equation}\label{eq:stationary_C_scalar}
    \bbC\bbC^\top
    =
    \bigl(\lambda+c h(\bbW)\bigr)\bbI,
\end{equation}
which implies that all singular values of $\bbC$ are equal.
Letting this common singular value be $\sigma$, and using the relation in \eqref{eq:det_to_exp} yields
%
\begin{equation}
    \sigma
    =
    e^{-h(\bbW)/d}.
\end{equation}
At the same time, we must have 
$\bbC\bbC^\top=\sigma^2\bbI$.
Thus, equating the scalar multiplying $\bbI$ on both sides of \eqref{eq:stationary_C_scalar} gives
\begin{equation}
    e^{-2h(\bbW)/d}
    =
    \lambda+c h(\bbW).
\end{equation}
Thus, the condition $\nabla_{\bbC}\bar{L}_c(\bbC,\lambda) = \bbzero$ implies
\begin{equation}\label{eq:scalar_stationarity_condition}
    \phi'_{\lambda,c}\bigl(h(\bbW)\bigr)=0,
\end{equation}
where $\phi_{\lambda,c}$ is the scalar function in \cref{lem:population_lower_bound}, whose derivative is given in \eqref{eq:derivative_phi}. Next, we analyze the scalar equation \eqref{eq:scalar_stationarity_condition}.

Since $\bbW\in\ccalW_1$, \cref{prop:nonneg_logdet} gives $h(\bbW)\geq 0$.
For $t\geq 0$,
\begin{equation}
    \phi''_{\lambda,c}(t)
    =
    \frac{2}{d}e^{-2t/d}
    +
    c
    >
    0,
\end{equation}
so $\phi'_{\lambda,c}$ is strictly increasing on $[0,\infty)$.
If $\lambda>1$, then $\phi'_{\lambda,c}(0)=\lambda-1>0$.
Therefore, $\phi'_{\lambda,c}(t)>0$ for $t\geq 0$, and \eqref{eq:scalar_stationarity_condition} has no solution.
This establishes claim (i) in \cref{thm:population_stationary_points}: no classical stationary point exists in $\ccalW_1$ when $\lambda>1$.

Consider now the case $\lambda=1$.
Then $\phi'_{1,c}(0)=0$, and strict monotonicity implies that the only solution of \eqref{eq:scalar_stationarity_condition} on $[0,\infty)$ is $h(\bbW)=0$.
Substituting $h(\bbW)=0$ into \eqref{eq:stationary_C_scalar} gives $\bbC\bbC^\top=\bbI$. Therefore,
\begin{equation}
    \bar{F}(\bbW)
    =
    \frac{1}{2}\tr(\bbC^\top\bbC)
    =
    \frac{d}{2}.
\end{equation}
Since $h(\bbW)=0$, \cref{prop:nonneg_logdet} implies that $\bbW$ is a DAG. Thus, $\bbW$ is feasible for the population DAG problem in \eqref{eq:population_identifiability}.
Moreover, $\bar{F}(\bbW)=d/2=\bar{F}(\bbW_0)$, where the last equality follows from the SEM at the true DAG and $\mbE[\bbz\bbz^\top ] = \bbI$.
By the identifiability condition in \eqref{eq:population_identifiability}, we must have $\bbW=\bbW_0$.

Conversely, at $\bbW=\bbW_0$ we have $h(\bbW_0)=0$.
Using $\nabla_{\bbW}\bar{F}(\bbW_0) = -\bbM_0^{-\top}$, $\nabla h(\bbW_0) = \bbM_0^{-\top}$, we obtain
\begin{equation}
    \nabla_{\bbW}\bar{L}_c(\bbW_0,1)
    =
    \nabla_{\bbW}\bar{F}(\bbW_0)
    +
    \nabla h(\bbW_0)
    =
    \bbzero.
\end{equation}
For $\lambda=1$, $\bbW_0$ is thus the unique classical stationary point.
\end{proof}

\section{Proof of \cref{thm:acyclic_kkt_points}}\label{app:proof_acyclic_kkt_points}
\begin{proof}
The proof has two steps.
First, we use the KKT conditions to derive a necessary inequality relating the population score to the acyclicity residual.
Then, for acyclic KKT points, this inequality forces the score to match that of the true DAG, and identifiability yields $\bbW=\bbW_0$.

Let $\bbW\in\ccalW_1$ be a KKT point of $\bar{L}_c(\cdot,1)$ over $\ccalW_1$, and let $\bbM=\bbI-\bbW$.
Using $\nabla_{\bbW}\bar{F}(\bbW) = -\bbSigma_{\bbx}\bbM$ and $\nabla h(\bbW) = \bbM^{-\top}$, the gradient of the augmented Lagrangian is given by
\begin{equation}\label{eq:kkt_augmented_A_def}
    \bbG(\bbW) \! := \!
    \nabla_{\bbW}\bar{L}_c(\bbW, \!1)
    \!=\!
    -\bbSigma_{\bbx}\bbM
    \!+\!
    \left( \!1 \!+\! c h(\bbW)\right)\bbM^{-\top}.
\end{equation}
Since $\bbW$ satisfies the KKT conditions, it must satisfy
\begin{equation}\label{eq:kkt_augmented_A_conditions}
    \bbG(\bbW)\geq \bbzero,
    \qquad
    \bbW\circ\bbG(\bbW)=\bbzero.
\end{equation}

We next convert complementarity into a scalar inequality.
From $\bbW=\bbI-\bbM$ and \eqref{eq:kkt_augmented_A_conditions},
\begin{equation}
    0
    =
    \langle \bbW,\bbG(\bbW)\rangle
    =
    \tr(\bbG(\bbW))
    -
    \langle \bbM,\bbG(\bbW)\rangle.
\end{equation}
Hence, $\tr(\bbG(\bbW)) = \langle \bbM,\bbG(\bbW)\rangle$.
%
Substituting \eqref{eq:kkt_augmented_A_def} into this expression gives
\begin{align}
    \langle \bbM,\bbG(\bbW)\rangle
    &=
    -\langle \bbM,\bbSigma_{\bbx}\bbM\rangle
    +
    \left(1+c h(\bbW)\right)
    \langle \bbM,\bbM^{-\top}\rangle
    \nonumber\\
    &=
    -2\bar{F}(\bbW)
    +
    d\left(1+c h(\bbW)\right) \geq 0,
\end{align}
%
where the inequality follows from $\langle \bbM,\bbG(\bbW)\rangle=\tr(\bbG(\bbW))$ and the fact that $\bbG(\bbW)\geq\bbzero$ has a non-negative trace. Rearranging terms leads to
\begin{equation}\label{eq:kkt_augmented_lower_score_gap}
    \bar{F}(\bbW)
    \leq
    \frac{d}{2}\left(1+c h(\bbW)\right).
\end{equation}

We now specialize to acyclic KKT points.

If $h(\bbW)=0$, then \eqref{eq:kkt_augmented_lower_score_gap} gives $\bar{F}(\bbW)\leq d/2$, while \cref{lem:population_lower_bound} with $\lambda=1$ and $h(\bbW)=0$ gives $\bar{F}(\bbW)=\bar{L}_c(\bbW,1)\geq d/2$, so $\bar{F}(\bbW)=d/2$.
Since $h(\bbW)=0$, $\bbW$ is a DAG by \cref{prop:nonneg_logdet} and $\bar{F}(\bbW)=d/2=\bar{F}(\bbW_0)$ by \eqref{eq:true_dag_score}, so the identifiability condition in \eqref{eq:population_identifiability} yields $\bbW=\bbW_0$.
\end{proof}

\section{Proof of \cref{lem:lipschitz}}\label{app:proof_lipschitz}
\begin{proof}

We derive $R$ in the statement of~\cref{lem:lipschitz} by combining Lipschitz constants for each term of $\nabla_{\bbW} \bar{L}_c(\bbW,\lambda)$.
In particular, we have that
\begin{alignat}{3}&
    \big\|
        \nabla_{\bbW} \bar{L}_c(\bbW,\lambda)
        -
        \nabla_{\bbW} \bar{L}_c(\bbW',\lambda)
    \big\|_F
&\nonumber\\&
    ~\leq~
    \big\|
        \nabla_{\bbW} \bar{F}(\bbW)
        -
        \nabla_{\bbW} \bar{F}(\bbW')
    \big\|_F
&\nonumber\\&
    \quad
    +
    \lambda
    \big\|
        \nabla h(\bbW)
        -
        \nabla h(\bbW')
    \big\|_F
&\nonumber\\&
    \quad
    +
    c
    \big\|
        h(\bbW) \nabla h(\bbW)
        -
        h(\bbW')
        \nabla h(\bbW')
    \big\|_F
&\nonumber\\&
    ~\leq~
    \big\|
        \nabla_{\bbW} \bar{F}(\bbW)
        -
        \nabla_{\bbW} \bar{F}(\bbW')
    \big\|_F
&\nonumber\\&
    \quad
    +
    \big(
        \lambda + c h(\bbW')
    \big)
    \big\|
        \nabla h(\bbW)
        -
        \nabla h(\bbW')
    \big\|_F
&\nonumber\\&
    \quad
    +
    c
    | h(\bbW) - h(\bbW') |
    \| \nabla h(\bbW) \|_F,
\label{eq:lip_proof_pt1}&\end{alignat}
where the first inequality is the triangle inequality, and the second follows from adding and subtracting $h(\bbW')\nabla h(\bbW)$ in the third term.

We proceed by bounding each term above.
To this end, we define $\tau$ as a constant such that $\| \nabla h(\bbW) \|_2 = \| (\bbI - \bbW)^{-1} \|_2 \leq \tau$ for any $\bbW \in \overline{\ccalW}_1$.
We obtain $\tau$ for two cases.
First, observe that when $\kappa < 1$, we have that
\begin{align}
    \big\| (\bbI - \bbW)^{-1} \big\|_2
    = {} &
    \sigma_{\min}(\bbI - \bbW)^{-1}\nonumber\\
     \leq {} &
    (1-\|\bbW\|_2)^{-1}\nonumber \\
    \! \leq {} &
    (1-\kappa)^{-1},
\label{eq:tau_lt1}\end{align}
so we let $\tau = (1-\kappa)^{-1}$ when $\kappa < 1$.

Otherwise, if $\kappa \geq 1$, we instead obtain the bound using the definition of a matrix inverse $\bbA^{-1} = \det(\bbA)^{-1} {\rm adj}(\bbA)$, where ${\rm adj}(\bbA)$ denotes the matrix adjugate of $\bbA$.
More specifically, since $\rho(\bbW) < 1$ for any $\bbW \in \overline{\ccalW}_1$, we have that
\begin{alignat}{3}&
    \big\| (\bbI - \bbW)^{-1} \big\|_2
    =
    \det(\bbI - \bbW)^{-1}
    \big\|
    \det(\bbI - \bbW)
    (\bbI - \bbW)^{-1}
    \big\|_2
&\nonumber\\&
    \qquad\quad
    =
    \prod_{i=1}^d
    (1 - \mu_i(\bbW))^{-1}
    \big\|
    {\rm adj}(\bbI - \bbW)
    \big\|_2
&\nonumber\\&
    \qquad\quad
    \leq
    (1-\alpha)^{-d}
    \big\|
    {\rm adj}(\bbI - \bbW)
    \big\|_2.
\nonumber&\end{alignat}
Then, since the $i$-th singular value of ${\rm adj}(\bbA)$ is $\prod_{j\neq i} \sigma_j(\bbA)$,  we again apply the constraint $\| \bbW \|_2 \leq \kappa$ for
\begin{alignat}{3}&
    \big\| (\bbI - \bbW)^{-1} \big\|_2
    \leq
    (1-\alpha)^{-d}
    \big\|
    {\rm adj}(\bbI - \bbW)
    \big\|_2
&\nonumber\\&
    \qquad\quad
    =
    (1-\alpha)^{-d}
    \sigma_{\min}(\bbI - \bbW)^{-1}
    \prod_{j=1}^d
    \sigma_j(\bbI - \bbW)
&\nonumber\\&
    \qquad\quad
    \leq
    (1-\alpha)^{-d}
    \| \bbI - \bbW \|_2^{d-1}
    \leq
    \frac{(1 + \kappa)^{d-1}}{(1 - \alpha)^{d}}
    =: \tau,
\label{eq:tau_ge1}&\end{alignat}
where we bound the product of singular values by canceling $\sigma_{\min}(\bbI - \bbW)$ and bounding the remaining singular values by $\|\bbI - \bbW\|_2 \leq 1 + \|\bbW\|_2 \leq 1+\kappa$, with the final bound coming from the triangle inequality.

With~\eqref{eq:tau_lt1} and~\eqref{eq:tau_ge1} defining $\tau$, we next obtain upper bounds for each term in~\eqref{eq:lip_proof_pt1} in terms of $\tau$.
By the definition of $\bar{F}(\bbW)$, 
for any $\bbW, \bbW' \in \overline{\ccalW}_1$ we have that
\begin{alignat}{3}
    \big\|
        &\nabla_{\bbW} \bar{F}(\bbW)
        -
        \nabla_{\bbW} \bar{F}(\bbW')
    \big\|_F
    ~=~
    \big\| \bbSigma_{\bbx} (\bbW - \bbW') \big\|_F
&\nonumber\\
    &~\leq~
    \|\bbSigma_{\bbx}\|_2
    \|\bbW - \bbW'\|_F
    ~\leq~
    \big\|
        (\bbI - \bbW_0)^{-1}
    \big\|_2^2
    \|\bbW - \bbW'\|_F
&\nonumber\\&
    ~\leq~
    \tau^2 
    \|\bbW - \bbW'\|_F.
\label{eq:fidelity_lip}&\end{alignat}

For the second term in~\eqref{eq:lip_proof_pt1}, we first note that
\begin{alignat}{3}&
    h(\bbW)
    =
    -\log\det(\bbI - \bbW)
    =
    \log \prod_{i=1}^d \mu_i(\bbI - \bbW)^{-1},
\nonumber
&\end{alignat}
where $\mu_i(\bbA)$ denotes the $i$-th eigenvalue of $\bbA$.
Since $|\mu_i(\bbW)| \leq \rho(\bbW) \leq \alpha < 1$, we have that
\begin{alignat}{3}&
    h(\bbW)
    &~\leq~&
    \log \prod_{i=1}^d \big|\mu_i(\bbI - \bbW)\big|^{-1}
    ~\leq~
    \log \prod_{i=1}^d (1 - |\mu_i(\bbW)|)^{-1}
&\nonumber\\&
    &~\leq~&
    \log \prod_{i=1}^d (1 - \alpha)^{-1}
    ~\leq~
    -d\log (1-\alpha).
\label{eq:acyc_bound}
&\end{alignat}
Then, we characterize the Lipschitz smoothness of $h(\bbW)$ via
\begin{alignat}{3}&
    \big\|
        \nabla h(\bbW) - \nabla h(\bbW')
    \big\|_F
&\nonumber\\&
    \qquad\qquad
    =
    \big\|
        (\bbI - \bbW)^{-1}
        -
        (\bbI - \bbW')^{-1}
    \big\|_F
&\nonumber\\&
    \qquad\qquad
    \leq
    \big\|
        (\bbI - \bbW)^{-1}
    \big\|_2
    \big\|
        (\bbI - \bbW')^{-1}
    \big\|_2
    \| \bbW - \bbW' \|_F
&\nonumber\\&
    \qquad\qquad
    \leq
    \tau^2
    \| \bbW - \bbW' \|_F,
\label{eq:acyc_lip}
&\end{alignat}

For the third term of~\eqref{eq:lip_proof_pt1}, we need a bound on $| h(\bbW) - h(\bbW') |$.
To this end, we require the following lemma.

\begin{lemma}\label{lem:line_seg}
    For any $\bbW \in \overline{\ccalW}_1$, let $\bbU \in \reals_+^{d\times d}$ be any matrix such that $\bbzero \leq \bbU \leq \bbW$ entrywise, and let
    \begin{alignat}{3}&
        \psi_t(\bbU,\bbV) = (1-t) \bbU + t \bbV ~\forall\: \bbU,\bbV \in \reals^{d\times d}.
    \label{eq:interp_function}
    &\end{alignat}
    For every $t\in[0,1]$, we have that $\psi_t(\bbU,\bbW) \in \overline{\ccalW}_1$, that is, the line segment $[\bbU,\bbW] := \{ \psi_t(\bbU,\bbW) \}_{t\in[0,1]} \subseteq \overline{\ccalW}_1$, and $h(\bbU) \leq h(\bbW)$.
\end{lemma}

\begin{proof}
    Since $\bbV_t := \psi_t(\bbU,\bbW)$ satisfies $\bbzero \leq \bbV_t \leq \bbW$ for any $t \in [0,1]$, the monotonicity of the spectral radius and spectral norm on non-negative matrices implies that $\bbV_t \in \overline{\ccalW}_1$.

    Since $\bbV_t \in \overline{\ccalW}_1$, $\bbI - \bbV_t$ is an $M$-matrix, so $\nabla h(\bbV_t) = (\bbI - \bbV_t)^{-\top} \geq \bbzero$, as $(\bbI - \bbV_t)^{-1} = \sum_{k \geq 0} \bbV_t^k$ by the Neumann series  (recall $\rho(\bbV_t) < 1$).
    We can then restrict $h$ to $[\bbU,\bbW]$ via the continuous mapping $g(t) = h( \bbV_t )$ with $t \mapsto \bbV_t$ for $t\in[0,1]$.
    By the chain rule, we take the derivative of $g$ as 
    \begin{alignat}{3}&
        g'(t)
        &=&
        \frac{ \partial h(\bbV_t) }{ \partial t}
        =
        \bigg\langle
            \nabla h(\bbV_t),
            \frac{ \partial \bbV_t }{ \partial t}
        \bigg\rangle
        =
        \big\langle
            \nabla h(\bbV_t), \bbW-\bbU 
        \big\rangle,
    \nonumber
    &\end{alignat}
    which is the inner product of two non-negative matrices and hence $g'(t) \geq 0$.
    By the mean value theorem, there exists some $t^* \in (0,1)$ such that $h(\bbW) - h(\bbU) = g(1) - g(0) = g'(t^*) \geq 0$, and therefore $h(\bbW) \geq h(\bbU)$, as desired.
\end{proof}



We let $\bbU = \min( \bbW, \bbW' )$ elementwise and assume, without loss of generality, that $h(\bbW) \geq h(\bbW')$.
By Lemma~\ref{lem:line_seg}, $[\bbU, \bbW], [\bbU,\bbW'] \subseteq \overline{\ccalW}_1$ and $h(\bbU) \leq \min\big(h(\bbW), h(\bbW')\big)$, so the mean value theorem asserts that there exists $\bbV \in [\bbU,\bbW]$ such that $h(\bbW) - h(\bbU) = \langle \nabla h(\bbV), \bbW - \bbU \rangle$.
%
Moreover, as $\bbV \in \overline{\ccalW}_1$, we have that $\| \nabla h (\bbV) \|_F \leq \sqrt{d} \| (\bbI-\bbV)^{-\top} \|_2 \leq \sqrt{d} \tau$ by~\eqref{eq:tau_lt1}-\eqref{eq:tau_ge1}.
Therefore,
\begin{align*}
    |h(\bbW)-h(\bbW')|
    &\le h(\bbW)-h(\bbU)
    \le\sqrt{d}\,\tau\|\bbW-\bbU\|_F \\
    &\le\sqrt{d}\,\tau\|\bbW-\bbW'\|_F,
\end{align*}
where the last inequality follows from the definition of $\bbU$, since $\| \bbW - \bbU\|_F \leq \| \bbW - \bbW'\|_F$.
We combine this with $\| \nabla h(\bbW) \|_F \leq \sqrt{d}\tau$ for
\begin{alignat}{3}&
    c
    |h(\bbW) - h(\bbW')|
    \|\nabla h(\bbW) \|_F
&\nonumber\\&
    \quad
    ~\leq~
    c\sqrt{d}\,\tau
    \|\nabla h(\bbW)\|_F
    \|\bbW-\bbW'\|_F
&\nonumber\\&
    \quad
    ~\leq~
    cd\tau
    \|\nabla h(\bbW)\|_2
    \|\bbW-\bbW'\|_F
&\nonumber\\&
    \quad
    ~\leq~
    cd
    \tau^2 \| \bbW - \bbW' \|_F.
\label{eq:acyc_mvt}
&\end{alignat}

Finally, we substitute~\eqref{eq:fidelity_lip}-\eqref{eq:acyc_lip} and~\eqref{eq:acyc_mvt} into~\eqref{eq:lip_proof_pt1} to obtain
\begin{alignat}{3}&
    \big\|
        \nabla_{\bbW} \bar{L}_c(\bbW,\lambda)
        -
        \nabla_{\bbW} \bar{L}_c(\bbW',\lambda)
    \big\|_F
&\nonumber\\&
    \qquad
    ~\leq~
    \tau^2
    \Big(
        1 + \lambda + cd - cd\log(1-\alpha)
    \Big)
    \| \bbW - \bbW' \|_F.
\label{eq:lip_proof_pt2}
&\end{alignat}
Combining~\eqref{eq:lip_proof_pt2} with our definition of $\tau$ in~\eqref{eq:tau_lt1} and~\eqref{eq:tau_ge1} gives us the result in Lemma~\ref{lem:lipschitz}.
\end{proof}

\section{Proof of \cref{thm:convergence}}\label{app:proof_convergence}
\begin{proof}

The conclusion of Theorem~\ref{thm:convergence} is a consequence of~\cite[Theorem 3.2]{frankel2015splitting}, which provides conditions under which gradient descent methods converge to global minima.
To apply this result, first observe that our PGD steps in~\eqref{eq:pgd} are equivalent to applying proximal gradient descent to $\Psi(\bbW) := \bar{L}_c(\bbW,\lambda) + \iota_{\overline{\ccalW}_1}(\bbW)$, where $\iota_{\overline{\ccalW}_1}$ denotes the indicator function such that $\iota_{\overline{\ccalW}_1}\!(\bbW) = 0$ for $\bbW \in \overline{\ccalW}_1$ and $\iota_{\overline{\ccalW}_1}\!(\bbW) = +\infty$ otherwise.
Indeed, the proximal operator of the indicator function $\iota_{\overline{\ccalW}_1}(\bbW)$ is the projection of $\bbW$ onto $\overline{\ccalW}_1$.
Naturally, since $\bbW_0 \in \overline{\ccalW}_1$, then $\bbW_0$ is also the unique global minimizer of $\Psi(\bbW)$.

Theorem 3.2 of~\cite{frankel2015splitting} states that, if $\Psi(\bbW)$ is Lipschitz smooth over $\overline{\ccalW}_1$ and has the KL property at the global minimum $\bbW_0$, and if the sequence $\{ \bbW_{(\ell)} \}_{\ell=0}^{\infty}$ generated via~\eqref{eq:pgd} satisfies certain assumptions, 
then $\{\bbW_{(\ell)} \}_{\ell =0}^{\infty}$ has finite length and converges to $\bbW_0$.
We therefore establish that our assumptions imply those of~\cite[Theorem 3.2]{frankel2015splitting}, which are the following:
\begin{enumerate}[label=(\roman*)]
    \item 
        For any $\ell \in \naturals$ and some $a_{\ell} > 0$, 
        \begin{equation}\label{eq:suff_decr}
            \Psi\big(\bbW_{(\ell)}\big)
            -
            \Psi\big(\bbW_{(\ell + 1)}\big)
            \geq
            a_{\ell}
            \big\|
                \bbW_{(\ell+1)}
                -
                \bbW_{(\ell)}
            \big\|_F^2;
        \end{equation}
    \item 
        For any $\ell \in \naturals$ and some $b_{\ell} > 0$, there exists a subgradient of $\Psi(\bbW_{(\ell+1)})$, denoted $\bbU_{\ell + 1}$, such that
        \begin{equation}\label{eq:rel_err}
            b_{\ell}
            \big\|
                \bbU_{\ell + 1}
            \big\|_F
            \leq
            \big\|
                \bbW_{(\ell + 1)}
                -
                \bbW_{(\ell)}
            \big\|_F;
        \end{equation}
    \item 
    The sequences $\{ a_{\ell} \}_{\ell \in \naturals}$ and $\{ b_{\ell} \}_{\ell \in \naturals}$ satisfy
    $\inf_{\ell} a_{\ell} > 0$, $\sum_{\ell} b_{\ell} = +\infty$, and $\sup_{\ell} \frac{1}{a_{\ell}b_{\ell}} < +\infty$.
\end{enumerate}


For~\eqref{eq:suff_decr}, note that for any $\bbW \in \overline{\ccalW}_1$, we have that $\Psi(\bbW) = \bar{L}_c(\bbW,\lambda)$, so we proceed while letting $\bar{L}(\bbW) = \bar{L}_c(\bbW,\lambda)$ for simplicity.
By Lemma~\ref{lem:line_seg}, we have that the line segments $[\bbW,\bbU], [\bbU,\bbW'] \subseteq \overline{\ccalW}_1$ for $\bbU = \min(\bbW,\bbW')$.
Thus, by Taylor's theorem, on $[\bbW,\bbU]$ we have that
\begin{equation}
    \bar{L}(\bbU)
    -
    \bar{L}(\bbW)
    \!
    =
    \!\!
    \int_0^1 
    \!\!
    \big\langle
        \nabla
        \bar{L}\big(\bbW + t(\bbU - \bbW)\big),
        \bbU - \bbW
    \big\rangle
    dt,
\end{equation}
which implies that
\begin{alignat}{3}&
    \bar{L}(\bbU)
    -
    \bar{L}(\bbW)
    -
    \big\langle
        \nabla \bar{L}( \bbW ),
        \bbU - \bbW
    \big\rangle
&\nonumber\\&
    ~
    =
    \!\int_0^1
    \!\!
    \big\langle
        \nabla
        \bar{L}\big(\bbW \!+ t(\bbU \!-\! \bbW)\big)
        \!-\!
        \nabla \bar{L}(\bbW),
        \bbU - \bbW
    \big\rangle
    dt,
\nonumber
&\end{alignat}
and by H\"older's inequality, we have that
\begin{alignat}{3}&
    \bar{L}(\bbU) - \bar{L}(\bbW)
    -
    \big\langle
        \nabla \bar{L}( \bbW ),
        \bbU - \bbW
    \big\rangle
&\nonumber\\&
    \,
    \leq \!
    \|\bbU \!-\! \bbW \|_F
    \!\!  \int_0^1 \!\!
    \|
        \nabla
        \bar{L}\big(\bbW \!+\! t(\bbU\!-\!\bbW)\big)
        \!-\!
        \nabla \bar{L}(\bbW)
    \|_F
    dt
&\nonumber\\&
    \,
    \leq
    R
    \|\bbU - \bbW \|_F^2
    \int_0^1
    t dt
    =
    \frac{R}{2}
    \|\bbU - \bbW \|_F^2,
\nonumber
&\end{alignat}
where the final inequality is due to the $R$-Lipschitz smoothness of $\bar{L}$ along $[\bbW,\bbU]$.
We apply the same argument along $[\bbU,\bbW']$ for
\begin{equation*}
    \bar{L}(\bbW')-\bar{L}(\bbU)
    -\big\langle\nabla\bar{L}(\bbU),\bbW'-\bbU\big\rangle
    \leq\frac{R}{2}\|\bbW'-\bbU\|_F^2.
\end{equation*}

Moreover, $\bbzero\leq\bbU\leq\bbW$ by construction, and $0 \leq h(\bbU) \leq \min\big( h(\bbW), h(\bbW') \big)$ by Lemma~\ref{lem:line_seg}.
Furthermore, observe that
\begin{alignat}{3}&
    \nabla h(\bbW) - \nabla h(\bbU)
    &~=~&
    (\bbI-\bbW)^{-\top}
    -
    (\bbI-\bbU)^{-\top}
&\nonumber\\&
    &~=~&
    (\bbI-\bbW)^{-\top}
    (\bbW-\bbU)^{\top}
    (\bbI-\bbU)^{-\top}.
\nonumber
&\end{alignat}
Since $\bbU \in \overline{\ccalW}_1$, $\nabla h(\bbU) \geq \bbzero$, so this is the matrix product of three non-negative matrices, thus $\nabla h(\bbW) \geq \nabla h(\bbU) \geq \bbzero$.
Additionally, recall that $\bbSigma_{\bbx} = (\bbI - \bbW_0)^{-\top} (\bbI - \bbW_0)^{-1} \geq \bbzero$ and $\lambda, c\geq 0$, so the population versions of~\eqref{eq:grad_augmented_lagrangian} and~\eqref{eq:grad_score} yield 
$\nabla \bar{L}(\bbU) \leq \nabla \bar{L}(\bbW)$ entrywise.

With the above, we sum the two Taylor bounds along $[\bbW,\bbU]$ and $[\bbU,\bbW']$ for
\begin{alignat}{3}&
    \bar{L}(\bbW')-\bar{L}(\bbW)
    -\big\langle\nabla \bar{L}(\bbW),\bbW'-\bbW\big\rangle
&\nonumber\\&
    \quad\leq
    \frac{R}{2}\big(\|\bbW-\bbU\|_F^2+\|\bbW'-\bbU\|_F^2\big)
&\nonumber\\&
    \qquad+\big\langle\nabla \bar{L}(\bbU)-\nabla \bar{L}(\bbW),\bbW'-\bbU\big\rangle
&\nonumber\\&
    \quad\leq\frac{R}{2}\|\bbW'-\bbW\|_F^2,
\label{eq:smooth_cond}
&\end{alignat}
where the last inequality follows from $\nabla \bar{L}(\bbU)\leq\nabla \bar{L}(\bbW)$, $\bbW'-\bbU\geq\bbzero$, and the disjoint entrywise supports of $\bbW-\bbU$ and $\bbW'-\bbU$.
Since~\eqref{eq:smooth_cond} holds for any $\bbW,\bbW'\in\overline{\ccalW}_1$, we can verify~\eqref{eq:suff_decr} since, for any $\ell \in \naturals$,
\begin{alignat}{3}&
    \Psi(\bbW_{(\ell+1)}) - \Psi(\bbW_{(\ell)})
    =
    \bar{L}(\bbW_{(\ell+1)}) - \bar{L}(\bbW_{(\ell)})
&\nonumber\\&
    ~
    \leq
    \big\langle
        \nabla \bar{L}(\bbW_{(\ell)}),
        \bbW_{(\ell+1)} - \bbW_{(\ell)}
    \big\rangle
    +
    \frac{R}{2}
    \|
        \bbW_{(\ell+1)} - \bbW_{(\ell)}
    \|_F^2.
\nonumber
&\end{alignat}

By the proximal interpretation of the projection $\Pi_{\overline{\ccalW}_1}$, 
we know that $\bbW_{(\ell+1)}$ is a minimizer of the function
\begin{equation*}
    Q(\tbW) := \big\langle\nabla\bar{L}(\bbW_{(\ell)}),\tbW-\bbW_{(\ell)}\big\rangle
    +\frac{1}{2\eta}\|\tbW-\bbW_{(\ell)}\|_F^2
\end{equation*}
over $\overline{\ccalW}_1$.
Thus, $Q(\bbW_{(\ell+1)}) - Q(\bbW_{(\ell)}) \leq 0$, or equivalently,
\begin{equation*}
    \big\langle\nabla\bar{L}(\bbW_{(\ell)}),\bbW_{(\ell+1)}-\bbW_{(\ell)}\big\rangle
    \leq-\frac{1}{2\eta}\|\bbW_{(\ell+1)}-\bbW_{(\ell)}\|_F^2.
\end{equation*}
Given~\eqref{eq:smooth_cond} and $\Psi(\bbW_{(\ell)}) = \bar{L}(\bbW_{(\ell)})$ for all $\ell \in \naturals$,
we obtain
\begin{alignat}{3}&
    \Psi(\bbW_{(\ell+1)})-\Psi(\bbW_{(\ell)})
    \leq
    \left(\frac{R}{2} - \frac{1}{2\eta}\right)
    \|\bbW_{(\ell+1)}-\bbW_{(\ell)}\|_F^2,
\nonumber
&\end{alignat}
which verifies~\eqref{eq:suff_decr} with $a_{\ell}=\frac{1}{2\eta}-\frac{R}{2}>0$, since $0<\eta<1/R$.
To verify~\eqref{eq:rel_err}, define
\begin{equation*}
    \bbV_{\ell+1}:=
    \frac{1}{\eta}
    \Big(
        \bbW_{(\ell)} - \bbW_{(\ell + 1)}
    \Big)
    -\nabla\bar{L}(\bbW_{(\ell)}).
\end{equation*}
By optimality of the proximal step in~\eqref{eq:pgd}, $\bbV_{\ell+1}\in\partial\iota_{\overline{\ccalW}_1}(\bbW_{(\ell+1)})$, where $\partial$ is the subdifferential.
Thus, $\nabla\bar{L}(\bbW_{(\ell+1)})+\bbV_{\ell+1}$ is a subgradient of $\Psi$ at $\bbW_{(\ell+1)}$, and the triangle inequality and Lemma~\ref{lem:lipschitz} give
\begin{alignat}{3}&
    \big\|\nabla\bar{L}(\bbW_{(\ell+1)})+\bbV_{\ell+1}\big\|_F
    \leq(\eta^{-1}+R)\|\bbW_{(\ell+1)}-\bbW_{(\ell)}\|_F.
\nonumber
&\end{alignat}
Hence,~\eqref{eq:rel_err} holds with $b_{\ell}=(\eta^{-1}+R)^{-1}>0$.

Observe that $a_{\ell}$ and $b_{\ell}$ are constant with positive values for all $\ell \in \naturals$, so assumption (iii) is trivially satisfied.
Thus, we need only verify that $\Psi(\bbW)$ is a KL function, that is, satisfies the KL property at $\bbW_0$.
To this end, we show that $\Psi(\bbW)$ is definable in an o-minimal structure over $\overline{\ccalW}_1$, which indicates that $\Psi$ has the KL property at every point in $\overline{\ccalW}_1$.

\noindent{\bf Definable constraint set.}
We begin by showing that the indicator function $\iota_{\overline{\ccalW}_1}(\bbW)$ is semialgebraic, hence definable.
By the Tarski quantifier elimination theorem~\cite{attouch2010proximal}, indicator functions are definable if their constraint sets are definable, so we proceed by considering $\overline{\ccalW}_1$.
Since $\bbW \geq \bbzero$, the spectral radius $\rho(\bbW)$ is an eigenvalue of $\bbW$ by the Perron--Frobenius theorem, so $\rho(\bbW) \leq \alpha$ holds if and only if $\det(t\bbI - \bbW) \neq 0$ for every $t > \alpha$.
Similarly, $\|\bbW\|_2 \leq \kappa$ holds if and only if $\kappa^2\bbI - \bbW^\top\bbW$ is positive semidefinite, that is, $\bbx^\top(\kappa^2\bbI - \bbW^\top\bbW)\bbx \geq 0$ for every $\bbx \in \reals^d$.
Both conditions are polynomial (in)equalities quantified over $t$ and $\bbx$, so the sets they define are semialgebraic by the Tarski--Seidenberg theorem~\cite[Thm.~2.2.1 and Prop.~2.2.4]{bochnak1998real}.
Moreover, $\bbW\geq\bbzero$ consists of $d^2$ entrywise inequalities.
Thus, the set $\overline{\ccalW}_1$ is semialgebraic and therefore definable, so $\iota_{\overline{\ccalW}_1}(\bbW)$ is definable in an o-minimal structure.

\noindent{\bf Definable trace.}
Observe that $\bar{F}(\bbW) = \frac{1}{2}\tr\big( (\bbI - \bbW)^\top \bbSigma_{\bbx} (\bbI - \bbW) \big)$ is a polynomial of the entries of $\bbW$, and therefore, $\bar{F}$ is semialgebraic. 
Similarly to $\overline{\ccalW}_1$, the Tarski quantifier elimination theorem~\cite{attouch2010proximal} guarantees that real semialgebraic sets belong to an o-minimal structure, so $\bar{F}$ is definable on the same o-minimal structure as $\iota_{\overline{\ccalW}_1}(\bbW)$.

\noindent{\bf Definable acyclicity function.}
Recall that the determinant $\det(\bbI - \bbW)$ can be written as a polynomial of the entries of $\bbW$, which implies that it is a real analytic function in $\bbW$.
Furthermore, $\bbW$ is restricted to the compact set $\overline{\ccalW}_1$, which is connected since it is star-shaped with respect to $\bbzero$. Indeed, for any $\bbW \in \overline{\ccalW}_1$ and $t \in [0,1]$, we have $\bbzero \leq t\bbW \leq \bbW$, so $t\bbW \in \overline{\ccalW}_1$ by monotonicity of $\rho$ and $\|\cdot\|_2$ on non-negative matrices.
Since $\det(\bbI - \bbW)$ is continuous and $\overline{\ccalW}_1$ is compact and connected, the image of $\overline{\ccalW}_1$ under $\det(\bbI - \bbW)$ is a connected, compact set, in particular the interval $[\mu_-,\mu_+]$ for $0 < \mu_- \leq \mu_+$.
To see why $\mu_- > 0$, observe that for every $\bbW \in \overline{\ccalW}_1$, $\bbW \geq \bbzero$ and $\rho(\bbW) \leq \alpha < 1$ guarantee that $\bbI - \bbW$ is a non-singular $M$-matrix, and therefore $\det(\bbI - \bbW) > 0$.

Next, observe that $-\log x$ is a real analytic function over $(0,\infty)$.
Thus, for some $\rho_-, \rho_+$ such that $0 < \rho_- < \mu_- \leq \mu_+ < \rho_+$, $-\log x$ is real analytic over $(\rho_-, \rho_+)$.
Let $\ccalK$ be the image of $(\rho_-, \rho_+)$ under the inverse of $\det(\bbI - \bbW)$, that is, $\ccalK$ is the open set $\{ \bbW \in \reals^{d\times d}: \det(\bbI - \bbW) \in (\rho_-, \rho_+) \} \supset \overline{\ccalW}_1$, which exists since $\det(\bbI - \bbW)$ is continuous.
Then, since compositions of real analytic functions are also real analytic, so too is $-\log\det(\bbI - \bbW)$ over $\ccalK$.

Note that real analytic functions that are restricted to compact sets are definable~\cite{attouch2010proximal}.
Thus, $-\log\det(\bbI - \bbW)$ restricted to $\overline{\ccalW}_1$ is definable, that is, its graph belongs to an o-minimal structure.
Moreover, by~\cite{van1998tame,wilkie1996model}, there exists an o-minimal structure subsuming restricted analytic functions and semialgebraic sets, so all terms of $\Psi(\bbW)$ belong to the same o-minimal structure.

Finally, note that compositions of definable functions are definable~\cite{attouch2010proximal}.
Since $x^2$ is a polynomial, its graph is a semialgebraic set and therefore definable, so $h^2(\bbW)$ is also definable.

Since finite sums of definable functions are definable, the loss $\Psi(\bbW)$ is definable and hence it is a KL function~\cite{attouch2010proximal}.
Since all conditions of~\cite[Theorem 3.2]{frankel2015splitting} hold, proximal gradient descent applied to $\Psi(\bbW)$, or equivalently, PGD applied to $\bar{L}_c(\bbW,\lambda)$, yields the guarantees of Theorem~\ref{thm:convergence}, as desired.
\end{proof}

\bibliographystyle{IEEEtran}
\bibliography{myIEEEabrv,biblio}

\end{document}